\documentclass[final]{cvpr}
\usepackage{times}
\usepackage{epsfig}
\usepackage{graphicx}
\usepackage{amsmath}
\usepackage{amssymb}

\usepackage{helvet}
\usepackage{courier}
\usepackage{caption}
\usepackage{amsfonts}
\usepackage{algpseudocode}
\usepackage{algorithm}
\usepackage{multirow}
\usepackage{graphicx}
\usepackage{comment,paralist}
\usepackage{amsmath,amssymb} % define this before the line numbering.
\usepackage{color}
\usepackage{xcolor}
\usepackage{bbm}
\usepackage{wrapfig}
\usepackage{threeparttable}
\usepackage{amsthm}
\usepackage{colortbl, array}
\usepackage{hhline}
\usepackage{booktabs}

\usepackage[pagebackref=true,breaklinks=true,colorlinks,bookmarks=false]{hyperref}

\hypersetup{
    urlcolor=blue,
}

\graphicspath{{figure/}{Figure/}}
\newtheorem{theorem}{Theorem}
\newtheorem{proposition}{Observation}

\newcounter{example}[section]
\DeclareMathOperator*{\argmax}{arg\,max} % thin space, limits underneath in displays

\definecolor{rulecolor}{RGB}{0,71,171}
\definecolor{tableheadcolor}{RGB}{204,229,255}
\definecolor{Gray}{gray}{0.9}

\newcommand{\myrowcolour}{\rowcolor{tableheadcolor}}
\newcommand{\highest}[1]{\textcolor{blue}{\textbf{#1}}}

\newcommand{\topline}{ %
        \arrayrulecolor{rulecolor}\specialrule{0.1em}{\abovetopsep}{0pt}%
        \arrayrulecolor{rulecolor}}
\newcommand{\bottomline}{ %
        \arrayrulecolor{rulecolor} %
        \specialrule{\heavyrulewidth}{0pt}{\belowbottomsep}}%

\newcommand\blfootnote[1]{%
  \begingroup
  \renewcommand\thefootnote{}\footnote{#1}%
  \addtocounter{footnote}{-1}%
  \endgroup
}

\def\cvprPaperID{3899} % *** Enter the CVPR Paper ID here
\def\confYear{CVPR 2021}

\begin{document}

%%%%%%%%% TITLE
\title{Simpler Certified Radius Maximization by Propagating Covariances}

\author{Xingjian Zhen$^{\dag}$, \ \ Rudrasis Chakraborty$^{\ddag}$, \ \ Vikas Singh$^{\dag}$\\ 
$^\dag$University of Wisconsin-Madison $\quad$ $^\ddag$University of California, Berkeley \\
{\tt\small xzhen3@wisc.edu} \ \ {\tt\small rudrasischa@gmail.com} \ \ {\tt\small vsingh@biostat.wisc.edu}
}

% \author{Xingjian Zhen\\
% University of Wisconsin-Madison\\
% %Institution1 address\\
% {\tt\small xzhen3@wisc.edu}
% % For a paper whose authors are all at the same institution,
% % omit the following lines up until the closing ``}''.
% % Additional authors and addresses can be added with ``\and'',
% % just like the second author.
% % To save space, use either the email address or home page, not both
% \and
% Rudrasis Chakraborty\\
% University of California, Berkeley\\
% %First line of institution2 address\\
% {\tt\small rudrasischa@gmail.com}
% \and 
% Vikas Singh\\
% University of Wisconsin-Madison\\
% {\tt\small vsingh@biostat.wisc.edu}
% }

\maketitle

\pagestyle{empty} 
\thispagestyle{empty} 

%%%%%%%%% ABSTRACT
\begin{abstract}
  One strategy for adversarially training a robust model is to maximize its certified radius  -- the neighborhood around a given training sample for which the model's prediction remains unchanged. The scheme typically involves analyzing a ``smoothed'' classifier where one estimates the prediction corresponding to Gaussian samples in the neighborhood of each sample in the mini-batch, accomplished in practice by Monte Carlo sampling. In this paper, we investigate the hypothesis that this sampling bottleneck can potentially be mitigated by identifying ways to directly propagate the covariance matrix of the smoothed distribution through the network. To this end, we find that other than certain adjustments to the network, propagating the covariances must also be accompanied by additional accounting that keeps track of how the distributional moments transform and interact at each stage in the network. We show how satisfying these criteria yields an algorithm for maximizing the certified radius on datasets including Cifar-10, ImageNet, and Places365 while offering runtime savings on networks with moderate depth, with a small compromise in overall accuracy. We describe the details of the key modifications that enable practical use. Via various experiments, we evaluate when our simplifications are sensible, and what the key benefits and limitations are. 
\end{abstract}

\blfootnote{Code is available at \url{https://github.com/zhenxingjian/Propagating_Covariance}. An short video summary of this paper is available at \url{https://youtu.be/m1ya2oNf5iE}}

%%%%%%%%% BODY TEXT

\section{Introduction}
\label{sec:intro}

%Deep learning models continue to achieve impressive performance
%gains on a broad range of challenging problems.
The prevailing approach for evaluating the performance of a deep learning model
involved assessing its overall accuracy profile
on one or more benchmarks of interest.
But the realization that many models were
not robust to even negligible adversarially-chosen perturbations of
the input data \cite{szegedy2013intriguing,biggio2013evasion,ilyas2018black, carlini2017towards}, and
may exhibit highly unstable behavior \cite{bojarski2016end, lecuyer2019certified,mansour2018deep} has led to the emergence
of robust training methods (or robust models) that offer, to varying degrees, immunity
to such adversarial perturbations.
Adversarial training has emerged as a popular mechanism to train a given deep 
model robustly \cite{papernot2016limitations,tramer2017ensemble}. Each mini-batch of training examples shown to the model is supplemented 
with adversarial samples. It makes sense that if the model parameter updates
are based on seeing enough adversarial samples which cover the perturbation
space well, the model is more robust to such adversarial examples at test time \cite{goodfellow2014explaining,huang2015learning,madry2017towards}.
The approach is effective although it often involves paying a premium in terms of
training time due to multiple gradient calculations \cite{shafahi2019adversarial}. 
However, many empirical defenses can fail when the attack is stronger \cite{carlini2017adversarial,uesato2018adversarial,athalye2018robustness}. 

%% Given an image $x$ whose label is $y$ and the well-trained neural network $f_\theta$, we have $f_\theta(\mathbf{x})\approx y$. However, there are always possible for an attacker to find $\boldsymbol{\delta}$ such that $\mathbf{x}+\boldsymbol \delta$ looks similar for human while $f_\theta(\mathbf{x}+\boldsymbol{\delta})\neq y$ and in fact the difference can be significantly large. However, the robustness is crucially important in the real-world situation like self-driving cars \cite{bojarski2016end, lecuyer2019certified} and the diagnostic of diseases \cite{mansour2018deep}.

While ideas to improve the efficiency of adversarial training
continue to evolve in the literature,
a complementary line of work seeks to avoid adversarial sample generation
entirely. One instead derives a {\em certifiable robustness} guarantee for a given
model 
\cite{weng2018towards,wong2017provable,zhang2018efficient,mirman2018differentiable,zhang2019you,singh2018fast,balunovic2019certifying}. 
The overall goal is to provide guarantees that {\em no perturbation} within a
certain range will change the prediction of the network.
An earlier proposal, interval bound propagation (IBP) \cite{gowal2018effectiveness},
used convex relaxations at different layers of the network
to derive the guarantees. 
Unfortunately, the bounds tend to get very loose  
as the network depth increases, see Fig. \ref{three_methods} (a). Thus,
the applicability to large high resolution datasets remains under-explored at this time. 

\begin{figure*}[bt]
%\vspace{-1em}
    \setlength{\abovecaptionskip}{0.1cm}
    \setlength{\belowcaptionskip}{-0.2cm} 
        \centering
               \includegraphics[height=0.34\columnwidth]{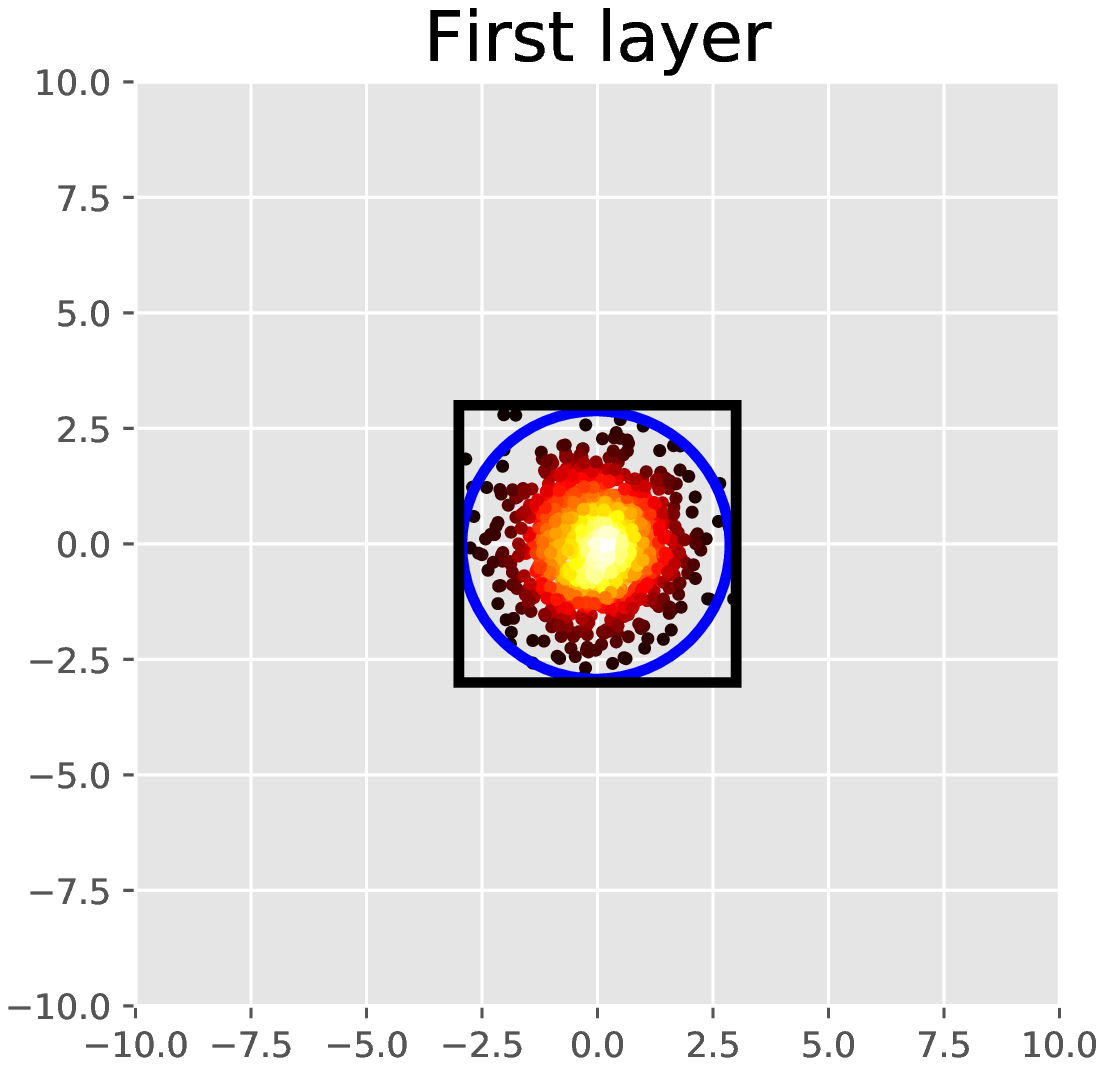}
               \includegraphics[height=0.34\columnwidth]{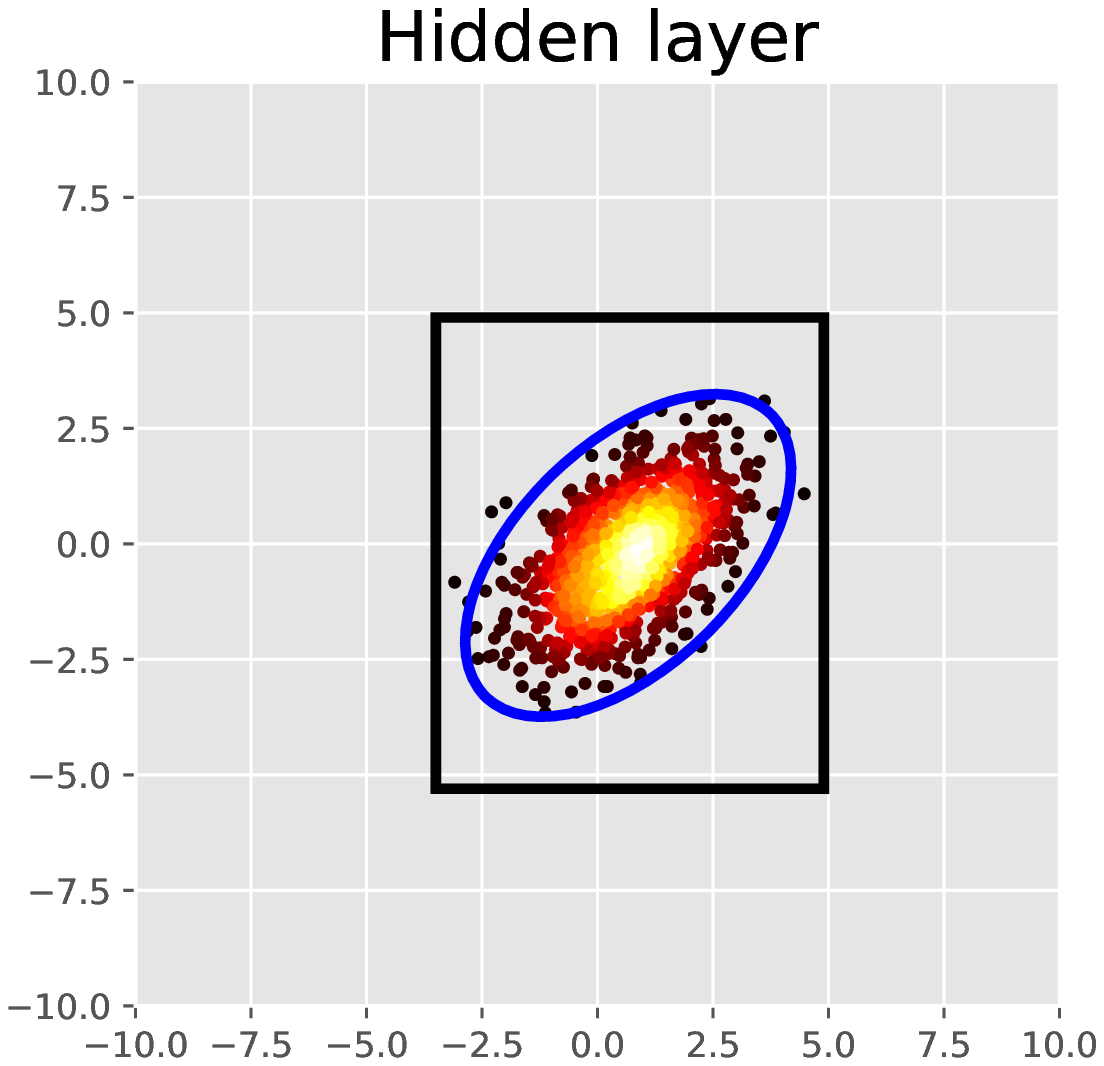}
               \includegraphics[height=0.34\columnwidth]{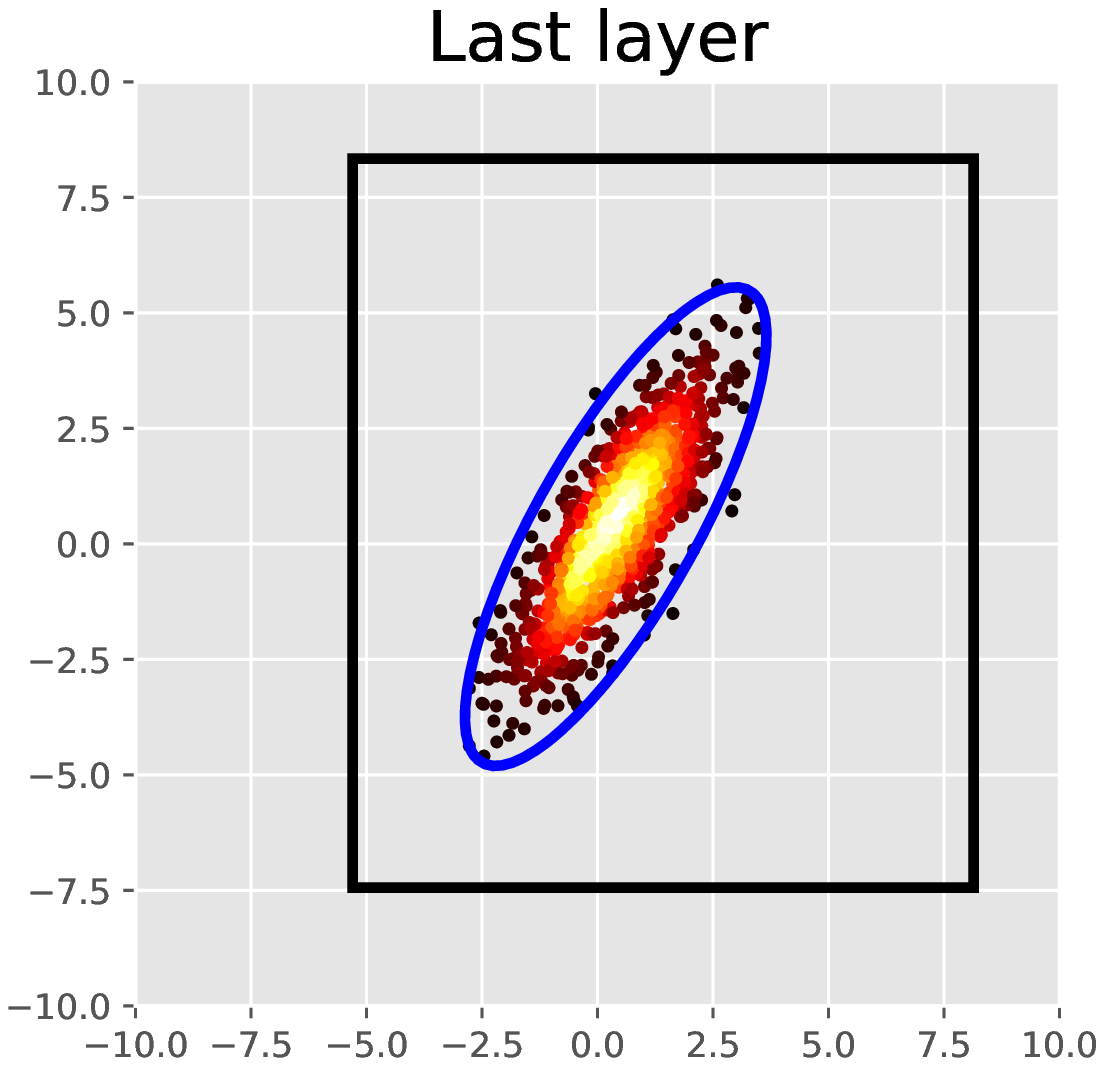}
               \includegraphics[height=0.36\columnwidth]{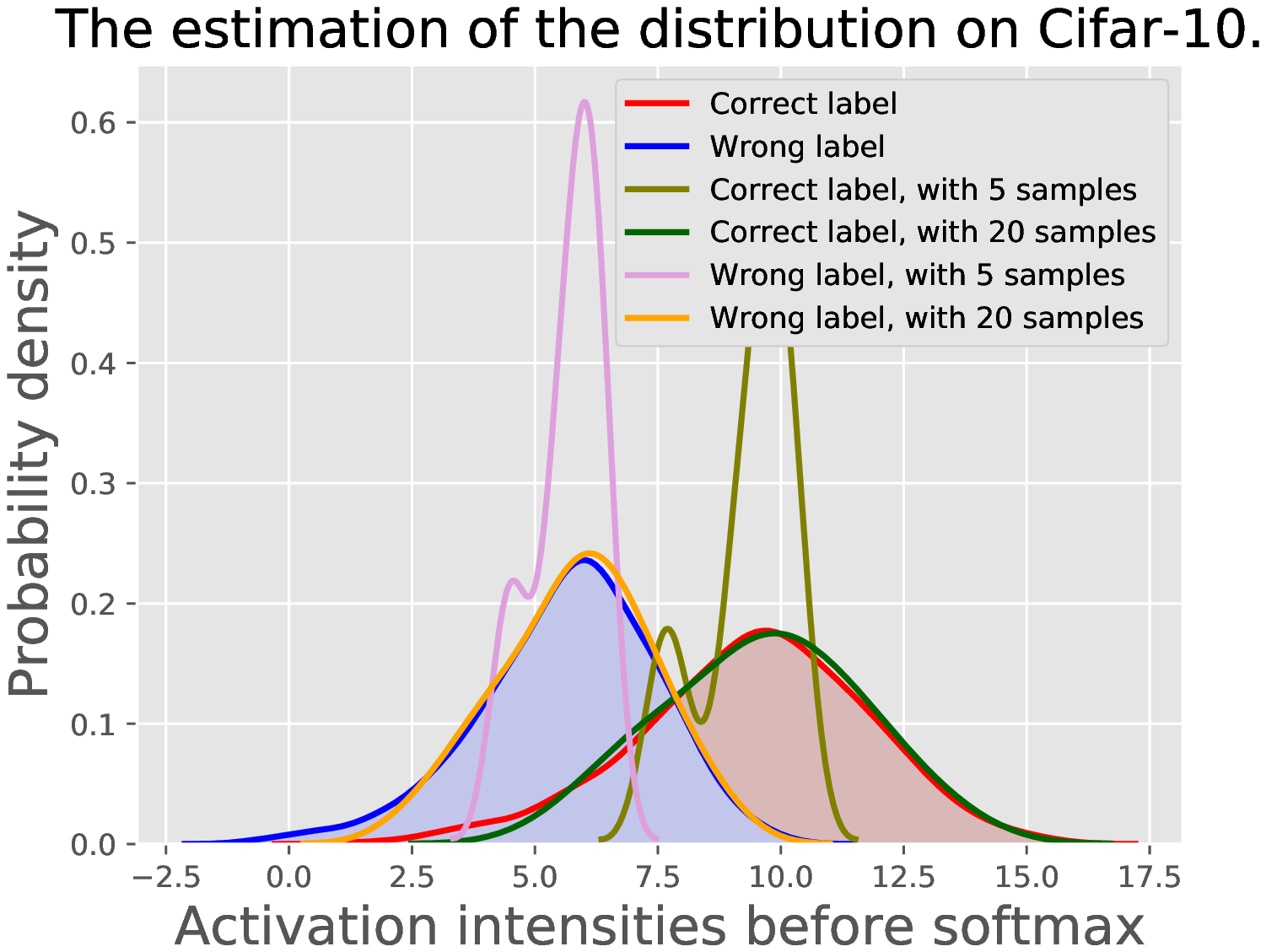}
               \includegraphics[height=0.36\columnwidth]{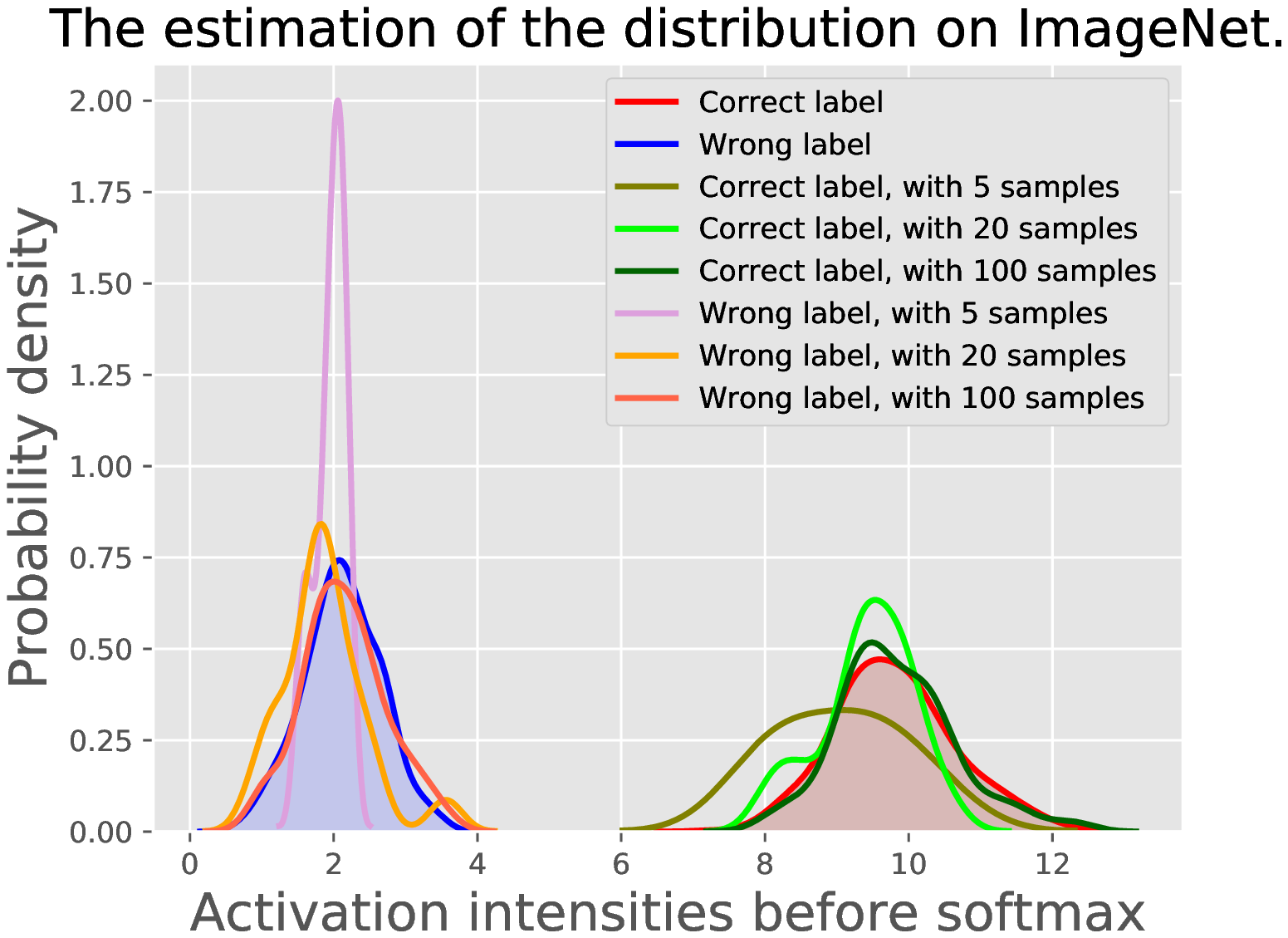}
               \caption{\label{three_methods}\footnotesize \textbf{(a) Columns 1-3:} Example of three methods for certifiable robustness on a
                  two layers MLP. We show results of the input layer, hidden layer, and the output layer here. 
                 Black boxes based on using IBP \cite{gowal2018effectiveness}. Red dots come
                 from the sampling idea from \cite{zhai2020macer}.
                 Ovals are covariance matrices if they are tracked exactly while considering
                 interactions. \textbf{(b) Columns 4-5:} Example of Monte Carlo estimation on a different dataset. If the distributions of the correct and wrong labels are farther, the network is more robust. As the size of images grows, the number of samples for a good estimate also increases.} 
                 \vspace*{-0.5em}
\end{figure*}

Recently, following the idea in \cite{li2018certified,lecuyer2019certified} at a high level, 
Cohen et al. \cite{cohen2019certified} introduced an interesting
randomized smoothing technique, which
can be used to certify the robust radius $C_R$. Assume that we have a 
base network $f_\theta(\cdot)$ for classification.
On a training image $\mathbf{x}\in \mathbb{R}^d$, the output
$f_\theta(\mathbf{x})\in \mathcal{Y}$ is the predicted label of the image $\mathbf{x}$.
Using $f_\theta(\cdot)$, we can build a
``smoothed'' neural network $g_\theta(\cdot)$.

{
\setlength{\abovedisplayskip}{0pt}
\setlength{\belowdisplayskip}{0pt}
\begin{align*}
	g_\theta(\mathbf{x})=\argmax_{c\in\mathcal{Y}}\mathbb{P}(f_\theta(\mathbf{x}+\boldsymbol{\varepsilon})=c),\text{where } \boldsymbol{\varepsilon} \sim \mathcal{N}(0,\sigma^2I)
\end{align*}
}

Here, $\sigma$ can be thought of as
a trade-off between the robustness and the
accuracy of the smoothed classifier $g_\theta(\cdot)$.
%Ideally, $g_\theta(\cdot)$ would predict whichever class the base classifier $f_\theta(\cdot)$ is most
%likely to predict,
%even when the input is perturbed. 
One can obtain 
a theoretical certified radius $C_R$ which states that
when $||\boldsymbol{\delta}||_2\leq C_R$, the classifier
$g_\theta(\mathbf{x}+\boldsymbol{\delta})$ will have same label
$y$ as $g_\theta(\mathbf{x})$. 
%However, $\mathbb{P}(f_\theta(\mathbf{x}+\boldsymbol{\varepsilon})=c)$ is hard to
%compute directly given the base classifier $f_\theta(\cdot)$.
MACER \cite{zhai2020macer} nicely extended these
ideas and also presented a differentiable form of randomized smoothing
showing how it enables maximizing the radius. 
Internally, a sampling scheme
%over $\boldsymbol{\varepsilon}$ to estimate that formula.
is used, where empirically, the number of samples to get an accurate estimation could be large.  
%However, Monte Carlo method highly depends on the number of samples.
As Fig. \ref{three_methods} (b) shows, one needs $100$ samples for a good estimation of the distribution of ImageNet. 

%What if we want to track the distribution directly?
%Take the simple two layer MLP as an example.

\noindent{\bf Main intuition:} MACER \cite{zhai2020macer} showed that by sampling from a Gaussian distribution and softening the estimation of the distribution in the last layer, maximizing the certified radius is feasible. It is interesting to ask if tracking the ``maximally perturbed'' distribution directly -- in the style of IBP -- is possible without sampling. Results in \cite{xiao2018dynamical} showed that the pre-activation vectors are i.i.d. Gaussian when the channel size goes to infinity. While unrealistic, 
it provides us a starting point. 
%Though this is unrealistic, in practice, the Gaussian %assumption can be sensible when the channel size is %relatively large. 
Since the Gaussian distribution can be fully characterized by the mean and the covariance matrix, 
we can track these two quantities 
as it passes through the network until the final layer, where the radius is calculated.
If implemented directly, this scheme must involve keeping track of how pixel correlations influence the entries of the covariances from one layer to the next,  and the bookkeeping needs grow rapidly. 
Alternatively, \cite{xiao2018dynamical} uses the fixed point of the covariance matrix 
to characterize it while it
passes through the network, but this idea is not adaptable 
for maximizing the radius task in \cite{zhai2020macer}. 
%which can be one way to simplify the tracking of the %covariance matrix from the input layer to the final layer. 
%However, due to the nature of the fixed point, the covariance is a %fixed value based on the input data, which neither shrinks nor %explodes within the neural network. 
%But to train a robust neural network, the goal, as we will see %shortly, is to minimize the covariance matrix, which in turn, %requires shrinking  the covariance matrix. 
We will use other convenient approximations of the 
covariance to make  directly tracking of the distribution of the perturbation feasible. 

\noindent{{\bf Other applications of certified radius maximization:}
Training a robust network is also useful when training in the presence of noisy labels \cite{angluin1988learning,goldberger2016training,patrini2017making}. Normally, 
%there are two cheap ways to curate a large dataset for %training deep models: 
both crowd-sourcing from non-experts and web annotations, 
common strategies for curating large datasets introduce noisy labels. 
%Since the neural network is able to memorize random labels easily \cite{zhang2016understanding}, it would be hard 
It can be difficult to train the model 
directly with the noisy labels without additional care 
\cite{zhang2016understanding}. 
Current methods either try to model the noise transition matrix \cite{goldberger2016training,patrini2017making}, or filter ``correct'' labels from the noisy dataset by collecting a consensus over different neural networks \cite{han2018co,jiang2017mentornet,malach2017decoupling,ren2018learning}. This leads us to consider {\it whether we can train the network from noisy labels without training any auxiliary network?} A key observation here is that the margin of clean labels should be smoother than the noisy labels (as shown in Fig. \ref{noisy_robust}). 
%Thus, if we can train a robust network cheaply, it %would be beneficial for training from noisy labels %as well.
}

\noindent{\bf Contributions:} This paper shows how several
known results characterizing the behavior of (and upper bounds on) covariance matrices
that arise from interactions between random variables with known covariance structure
can be leveraged to obtain a simple scheme that can propagate the distribution (perturbation
applied to the training samples) through the network. 
%\begingroup
\setlength{\intextsep}{8pt}%
\setlength{\columnsep}{6pt}%
\begin{wrapfigure}{r}{0.57\columnwidth}
\centering
%\vspace*{-0.5em}
          \includegraphics[width=0.55\columnwidth]{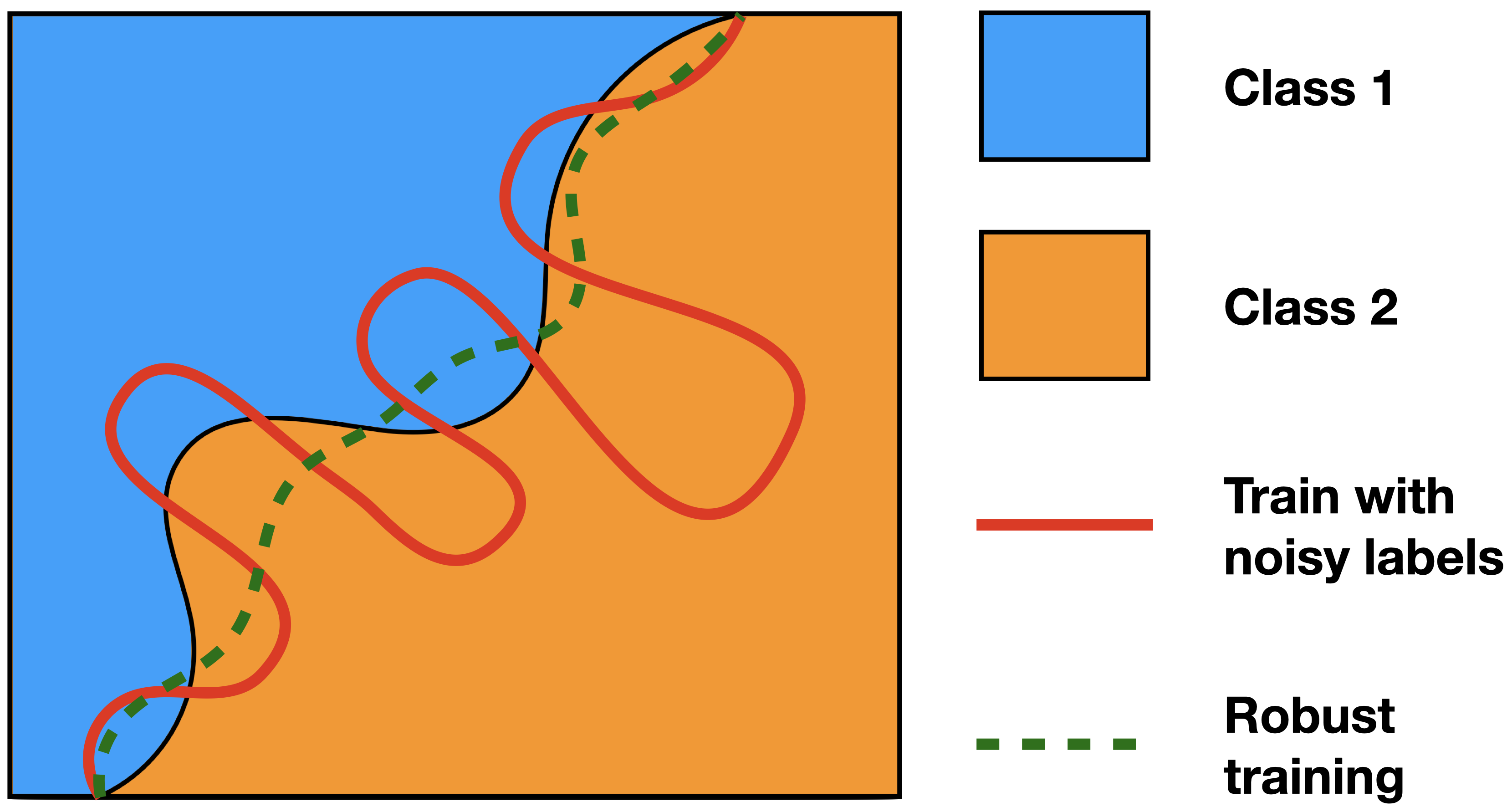}
               \caption{\footnotesize When directly training with noisy labels, the margin will resemble the red line. Using a robust network, the margin will resemble the green line. }\label{noisy_robust}
            \vspace*{-1em}
\end{wrapfigure}
%\endgroup
This leads to a sampling-free
method that performs favorably when compared to \cite{zhai2020macer} and other
similar approaches when the network depth is moderate. We show that our method is $5\times$ faster on Cifar-10 dataset and $1.5\times$ faster on larger datasets including ImageNet and Places365 relative to the current state-of-the-art without sacrificing much of the performance. Also, we show that the idea is applicable to  or training with noisy labels.

\setlength{\columnsep}{0.3125in}
\setlength{\intextsep}{12.0pt plus 2.0pt minus 2.0pt}

\section{Robust Radius Via Randomized Smoothing}
We will briefly review the
relevant background on robust radius calculation using  Monte-Carlo (MC) sampling.\\
\textbf{What is the robust radius?} In order to measure the robustness of a neural network,
the {\it robust radius} has been shown to be a sensible measure
\cite{weng2018towards,cohen2019certified}. Given a trained neural network $f_{\theta}$,
the {\it $\ell_2$-robustness at data point $(\mathbf{x}, y)$} is defined as the {\bf largest}
radius $R$ of the ball centered at $\mathbf{x}$ such that all samples within the ball will
be classified as $y$ by the neural network $f_{\theta}$.
Analogously, the {\it $\ell_2$-robustness of $f_{\theta}$} is defined as the
{\bf minimum} $\ell_2$-robustness at data point $(\mathbf{x}, y)$ over the dataset.
But calculating the robust radius for the neural network can be hard;
\cite{weng2018towards} provides a hardness result for the $\ell_1$-robust radius.
In order to make computing $\ell_2$-robustness tractable,
the idea in \cite{cohen2019certified} suggests
working with a tight lower bound, called the ``Certified Radius'',
denoted by $0\leq C_R\leq R$.
Let us now briefly review \cite{cohen2019certified} its functions and features  
for a given base classifier $f_\theta(\cdot)$.

Note that we want to certify that there will be {\em no adversarial samples}
within a radius of $C_R$. By smoothing out the perturbations $\boldsymbol{\varepsilon}$ around the input image/data $\mathbf{x}$ for the base classifier $f_\theta(\cdot)$, intuitively it will be
harder to find an adversarial sample, since it will actually
require finding a ``region'' of adversarial samples.
If we can estimate a lower bound on the probability of the base
  classifier to correctly classify the perturbed data $\mathbf{x}+\boldsymbol{\varepsilon}$,
  denoted as $\underline{p_{c_\mathbf{x}}}$, as well as an upper bound of the
  probability of an incorrect
  classification $\overline{p_{\widetilde{c}}}\leq 1-\underline{p_{c_\mathbf{x}}}$, where $c_\mathbf{x}$ is the true label of
  $\mathbf{x}$ and $\widetilde{c}$ is the ``most likely to be confused'' incorrect label,
  a nice result for the smoothed classifier $g_\theta(\cdot)$ is available,
%\vspace{-0.7em}
\begin{theorem}
    \label{certified_radius}
    \cite{cohen2019certified} Let $f_\theta:\mathbf{R}^d\rightarrow \mathcal{Y}$ be any deterministic or random function, and let $\boldsymbol{\varepsilon} \sim \mathcal{N}(\mathbf{0},\sigma^2I)$. 
Let $g_\theta$ be the randomized smoothing classifier defined as $g_\theta(\mathbf{x})=\argmax_{c\in\mathcal{Y}}\mathbb{P}(f_\theta(\mathbf{x}+\boldsymbol{\varepsilon})=c_{\mathbf{x}})$. Suppose $c_{\mathbf{x}},\widetilde{c}\in\mathcal{Y}$ and $\underline{p_{c_{\mathbf{x}}}},\overline{p_{\widetilde{c}}}\in [0,1]$ satisfy $\mathbb{P}(f_\theta(\mathbf{x}+\boldsymbol{\varepsilon})=c_{\mathbf{x}})\geq \underline{p_{c_\mathbf{x}}}\geq \overline{p_{\widetilde{c}}}\geq \max_{\widetilde{c}\neq c_{\mathbf{x}}}\mathbb{P}(f_\theta(\mathbf{x}+\boldsymbol{\varepsilon})=\widetilde{c})$. 
Then $g_\theta(\mathbf{x}+\boldsymbol{\delta})=c_{\mathbf{x}}$ 
for all $\|\boldsymbol{\delta}\|_2<C_R$, 
where $C_R=\frac{\sigma}{2}(\Phi^{-1}(\underline{p_{c_{\mathbf{x}}}})-\Phi^{-1}(\overline{p_{\widetilde{c}}}))$. 
\end{theorem} 
%\vspace{-0.3em}
The symbol $\Phi$ denotes the CDF of the standard Normal distribution. $\Phi$ and $\Phi^{-1}$
are involved because of smoothing the Gaussian perturbation $\boldsymbol{{\varepsilon}}$. The proof of this theorem can be found in \cite{cohen2019certified}. We also include it in the appendix.

\noindent \textbf{How to compute the robust radius?} Using Theorem \ref{certified_radius},
we will need to
compute the lower bound $\underline{p_{c_{\mathbf{x}}}}$,
the main ingredient to compute $C_R$. 
%
%In \cite{cohen2019certified}, the authors introduced a sampling-based method to compute the lower bound of $\underline{p_{c_{\mathbf{x}}}}$ in the testing phase, as shown in Alg. \ref{Alg:certify}. We first sample $n_0$ of noise samples around the input $\mathbf{x}$ and pass it through the base classifier $f_\theta$ to estimate the classified label after smoothing. Then we sample $n$ of noise samples, where $n\gg n_0$, to estimate the lower bound of $\underline{p_{c_{\mathbf{x}}}}$ with certain confident level $\alpha$. 
%The function LowerConfBound$(k,n,1-\alpha)$ in the Alg. \ref{Alg:certify} returns a one-side s$(1-\alpha)$ lower confidence interval for the Binomial parameter $p$ given a sample $k \sim Binomial(n,p)$. 
%
In \cite{cohen2019certified}, the authors introduced a sampling-based method to compute the lower bound of $\underline{p_{c_{\mathbf{x}}}}$
in the test phase. The procedure
first samples $n_0$ noisy samples around $\mathbf{x}$ and passes
it through the base classifier $f_\theta$ to estimate the classified label
{\em after} smoothing. Then,
we sample $n$ noisy samples, where $n\gg n_0$, to estimate the
lower bound of $\underline{p_{c_{\mathbf{x}}}}$ for
a certain confidence level $\alpha$. 

\section{Track Distribution Approximately}
In the last section, we discussed how to calculate $\underline{p_{c_{\mathbf{x}}}}$
in a sampling (Monte Carlo) based setting. However, this method is based on counting the number of
correctly classified samples, which is not differentiable during training.
In order to tackle this problem, \cite{zhai2020macer} introduced an 
alternative -- soft randomized smoothing --
to calculate the lower bound
\begin{align}
    \underline{p_{c_{\mathbf{x}}}} & =\mathbb{E}_{\boldsymbol{\varepsilon} \sim \mathcal{N}(\mathbf{0},\sigma^2I)}\left [\frac{e^{\beta u_\theta^{c_{\mathbf{x}}}(\mathbf{x}+\boldsymbol{\varepsilon})}}{\sum_{c'\in\mathcal{Y}}e^{\beta u_\theta^{c'}(\mathbf{x}+\boldsymbol{\varepsilon})}}\right]
\end{align}
where $u_{\theta}$ is the network $f_\theta$ {\bf without the last softmax layer}, i.e., $f_\theta=\argmax \text{softmax}(u_\theta)$, and
$\beta$ is a hyperparameter. 

From Fig. \ref{three_methods} (b), observe that if we have enough MC samples, we can reliably estimate $\underline{p_{c_{\mathbf{x}}}}$
effectively by counting the number of correctly classified samples.
If we can bypass MC sampling to estimate the final distribution, the gains in runtime can be significant. However, directly computing the joint distribution of the perturbations of all the pixels is infeasible: we need simplifying assumptions. 

\noindent{\bf Gaussian pre-activation vectors:} The first assumption is to use a Gaussian distribution to fit the pre-activation vectors. As briefly mentioned before, this is true when the channel size goes to infinity by the central limit theorem. In practice, when the channel size is large enough, e.g., a ResNet-based architecture \cite{he2016deep}, this assumption may be acceptable with a small error (evaluated later in experiments).
Therefore, we will only consider the first two moments, which is reasonable for Gaussian perturbation \cite{wishart1928generalised}. 

\noindent{\bf Second moments:} Our second assumption is that in each layer of a convolution network, the second moments are identical for the {\em perturbation} of all pixels. The input pixels share identical second moments from a fixed Gaussian perturbation $\varepsilon$. Due to weight sharing and the linearity of the convolution operators, the second moments will only depend on the kernel matrix without the position information. A more detailed discussion is in Obs. \ref{prop:identical} and the appendix.

\noindent\textbf{Notations and setup:} 
Let $N$ be the number of channels. 
We use
$\Sigma$ as the covariance matrices of the perturbation across the channels unless otherwise noted. The input perturbation comes from Gaussian perturbation $\varepsilon$ where $\Sigma=\sigma^2 I$. As the image passes through the network, the input perturbation directly influences the output at each pixel as a function of the network parameters.
We use $\Sigma_i \in\mathbf{R}^{N\times N}$, shorthand for $\Sigma_{\mathbf{x}_i}$, to denote the covariance of the perturbation
distribution associated with pixel $i$ of image $\mathbf{x}$ denoted as $\mathbf{x}_i$.
We call $\Sigma[i,j]$ as the $(i,j)$-entry of $\Sigma$. Notice that the $N$ changes from one layer to the other as the number of channels are different. So, the size of $\Sigma$ will change.  
Let $M_q$ be the number of pixels in the $q^{th}$ layer input,
i.e., 
%for $q=0$, $M_0$ is the number of pixels in the input layer and
for $q=1$, $M_1$ is the number of pixels in the $1^{\text{st}}$ hidden layer of the network.

Similarly, $\boldsymbol{\mu}_{\mathbf{x}_i}$ or $\boldsymbol{\mu}_{i} \in \mathbf{R}^{N}$ is the mean of the distribution of the pixel $\mathbf{x}_i$ 
intensity after the perturbation. 
In the input layer, since the perturbation $\varepsilon\sim \mathcal{N}(0,\sigma^2 I)$, $\boldsymbol{\mu}_i=\mathbf{x}_i$. 
At the $u_\theta$ layer, the number of channels is the number of classes, with the 
number of pixels being $1$. 
We use $\boldsymbol{\mu}[c_{\mathbf{x}}]$ and $\Sigma[c_{\mathbf{x}},c_{\mathbf{x}}]$ to denote the $c_{\mathbf{x}}$ component of $\boldsymbol{\mu}$ and $(c_{\mathbf{x}}, c_{\mathbf{x}})$-entry of $\Sigma$ respectively. To denote the cross-correlation between two pixels $\mathbf{x}_i, \mathbf{x}_j$, we
use $E_{\mathbf{x}_i \mathbf{x}_j}$ or $E_{ij} \in \mathbf{R}^{N\times N}$.
Note that this cross-correlation is across channels.  For the special case with
channel size $N=1$, we will
use $\sigma^{(i)}\in \mathbf{R}$ to represent the variance in the $i^{\text{th}}$ layer. 
Let us define, 
\begin{align}
c_{\mathbf{x}}=\displaystyle\argmax_{c\in\mathcal{Y}} \boldsymbol{\mu}[c], \quad \widetilde{c}=\displaystyle\argmax_{{c}\in\mathcal{Y},{c}\neq c_{\mathbf{x}}} \boldsymbol{\mu}[c]
\end{align}

Let the number of classes $C = |\mathcal{Y}|$. Then, we can state the following. 
\begin{proposition}
	\label{Gaussian_estimate}
	Using $u_{\theta}$, the prediction of the model
        can be written as  $f_\theta(\mathbf{x})=\argmax_{c\in\mathcal{Y}}\text{ softmax}(u_\theta(\mathbf{x}))$. Assume  $u_\theta(\mathbf{x}) \sim \mathcal{N}(\boldsymbol{\mu},\Sigma)$, where $\mathbf{x} \sim \mathcal{N}(\boldsymbol{\mu}_{\mathbf{x}},\Sigma_{\mathbf{x}})$, $ \boldsymbol{\mu} \in \mathbf{R}^C$ and $\Sigma \in \mathbf{R}^{C\times C}$. Then the estimation of $\underline{p_{c_{\mathbf{x}}}}$ is
        {
		\setlength{\abovedisplayskip}{3pt}
		\setlength{\belowdisplayskip}{0pt}
        \begin{align}\underline{p_{c_{\mathbf{x}}}} &= \Phi(\frac{\boldsymbol{\mu}[{c_{\mathbf{x}}}]-\boldsymbol{\mu}[{\widetilde{c}}]}{\sqrt{\Sigma[{c_{\mathbf{x}},c_{\mathbf{x}}}]+\Sigma[{\widetilde{c},\widetilde{c}}]-2\Sigma[{c_{\mathbf{x}},\widetilde{c}}]}})\end{align}
        }
        \end{proposition}
        \vspace{-0.2em}

  Notice that propagating $\boldsymbol{\mu}$ through the network is simple, since tracking the mean is the same as directly passing it through the network when there is no nonlinear activation, and requires no cross-correlation between pixels. But tracking $\Sigma$ at each step of the network can be challenging and some approximating techniques have been used in literature for simple networks \cite{scaglione2008decentralized}.
  To see this, let us consider a simple 1-D example. 
% \vspace{-0.7em}
%\begin{example}
%\label{1dexample}
 
 \noindent{\bf Bookkeeping problem:} 
 Consider a simple 1-D convolution with a kernel size $k$. By Obs. \ref{Gaussian_estimate}, we will need the distribution of  $u_\theta(\mathbf{x})$ of the $i^{th}$ layer (i.e., the network without the softmax layer).
 Directly, this will involve taking into 
 account $k^1$ pixels in the $(i-1)^{th}$ layer, and $k^2$ pixels in the $(i-2)^{th}$ layer. 
 We must calculate the covariance $\Sigma$ and also calculate the cross-correlation $E$ between all $k^q$ pixels in $(i-q)^{th}$ layer.  
%\end{example}
% \vspace{-0.2em}
{
This trend stops when we 
hit $k^q > M_{i-q}$, where $M_{i-q}$ is the number of pixels at $(i-q)$ layer, but 
it is impractical anyway. 

If we temporarily assume that the network involves no activation functions, 
and if the input perturbation is identical 
for all pixels, then the variance of all 
pixels after perturbation is also 
identical. Thus, the variance of each pixel only relies on the variance of the perturbation and  
not  on the pixel intensity.
This may allow us to track one covariance matrix instead of $M$ for all $M$ pixels.} 

\vspace{-0.2em}
\begin{proposition}
\label{prop:identical}
	With the input perturbation $\boldsymbol{\varepsilon}$ set to be identical along the spatial dimension and without nonlinear activation function, for $q^{th}$ hidden convolution layer with $\{\mathbf{h}_i\}_{i=1}^{M_q}$ output pixels, we have $\Sigma_{\mathbf{h}_i}^{(q)} = \Sigma_{\mathbf{h}_j}^{(q)},\forall i, j\in \{1, \cdots, M_q\}$.
\end{proposition}
\vspace{-0.3em}

The Obs. \ref{prop:identical}
only reduces the cost marginally:
instead of computing all the covariances of the perturbation for all pixels, $\Sigma_1^{(i-q)}, \Sigma_2^{(i-q)},\cdots,\Sigma_k^{(i-q)}$, 
we only need to compute a single $\Sigma^{(i-q)}$. 
Unfortunately, we still need to compute all different $E_{ij}^{(i-2)}$ that will contribute to $u_\theta(\mathbf{x})$  (also see worked out example in the appendix).
Thus, due to these cross-correlation terms $E_{ij}$, the overall computation is still not feasible.
In any case, the assumption itself is unrealistic: we {\em do} need to 
take nonlinear activations into account 
which will break the identity assumption of the second moments. 
For this reason, we explore a useful 
approximation 
which we discuss next. 

\subsection{How To Make Distribution Tracking Feasible}
From the previous discussion, we observe that a key bottleneck of tracking distribution across layers is to track the interaction
between pairs of pixels, i.e., cross-correlations. Thus, we need an estimate of
the cross-correlations between pixels.
In \cite{hanebeck2001tight}, the authors provide an upper-bound on the
joint distribution of two multivariate Gaussian random variables such that the 
upper bounding
distribution contains {\bf no cross-correlations}. This result will be crucial for us. 

Formally, let $\mathbf{x}_1, \mathbf{x}_2\in\mathbf{R}^N$ be  two random vectors representing two pixels with $N$ channels. Without any loss of generalization, assume that $\mathbf{x}_1\sim\mathcal{N}(\mathbf{0}, \Sigma_{1})$, and $\mathbf{x}_2\sim\mathcal{N}(\mathbf{0},\Sigma_{2})$ (if the mean is not $\mathbf{0}$, we can subtract the mean without affecting the covariance matrix).  Also, assume that we do not know the cross-correlation between $\mathbf{x}_1$ and $\mathbf{x}_2$, i.e., ${E}_{12}$. Instead, the correlation coefficient $r$ is bounded by $r_{max}$, i.e., $|r|\leq r_{max}$.  

With the above assumptions, we can
bound the covariance matrix of the joint distribution of two $N$-dimensional
random vectors ${\mathbf{x}_1},{\mathbf{x}_2}$
by two independent random vectors $\widehat{\mathbf{x}_1},\widehat{\mathbf{x}_2}$.
We will use the notation `` $\widehat{\cdot}$ '' to denote the upper bound estimation of ``$\cdot$''. The upper bound here means that $[\widehat{\Sigma}-\Sigma]$ is a positive semi-definite matrix, where $\widehat{\Sigma}$ is the joint distribution of the two independent random vectors $\widehat{\mathbf{x}_1}$ and $\widehat{\mathbf{x}_2}$. Here, $\Sigma$ is the joint distribution of ${\mathbf{x}_1},{\mathbf{x}_2}$ with correlation. Formally, 
%formal statement is presented in the following %theorem. 
%\vspace{-0.7em}
\begin{theorem}
    \label{tight_bound}
    \cite{hanebeck2001tight} When $\widehat{\mathbf{x}_1}\sim \mathcal{N}(\mathbf{0}, \widehat{\Sigma_{1}} = \tau_1\Sigma_{1})$, and $\widehat{\mathbf{x}_2}\sim \mathcal{N}(\mathbf{0}, \widehat{\Sigma_{2}} = \tau_2\Sigma_{2})$, the covariance matrix $ \mathbf{B} = \widehat{\Sigma} = \begin{bmatrix} \tau_1\Sigma_{1} & 0 \\ 0 & \tau_2\Sigma_{2}\end{bmatrix}$ bounds the joint distribution of $\mathbf{x}_1$ and $\mathbf{x}_2$, i.e., $\mathbf{B} \succeq  \Sigma =  \begin{bmatrix} \Sigma_{1} & E_{12} \\ E_{21} & \Sigma_{2}\end{bmatrix}$, where $\tau_1=\frac{1}{\eta-\kappa}$, $\tau_2=\frac{1}{\eta+\kappa}$, $\kappa^2\leq \frac{1-2\eta}{1-r_{max}^2}+\eta^2$, and $0.5\leq \eta\leq \frac{1}{1+r_{max}}$.
\end{theorem}
%\vspace{-0.7em}
With this result in hand, we now discuss how to use it to
makes the tracking of moments across layers feasible. \\
\textbf{How to use Theorem \ref{tight_bound}?} 
By Obs. \ref{prop:identical}, we can store {\it one covariance matrix} over the convolved output pixels at each layer.
%Recall that \begin{inparaenum}[\bfseries (a)] \item We assume the covariance matrices of the input layer are identical across pixels. Thus, we can store only {\it one covariance matrix}, which
%  radically reduces the memory requirement. \item After a convolution layer, each convolved
%  result at the pixel is also identical due to the sharing of the kernel if substrating the mean.
%  Hence we can store {\it one covariance matrix} over the convolved output pixels. \end{inparaenum} 
%
Notice that due to the presence of the cross-correlation between output pixels,
we also need to store cross-correlation matrices, which was our bottleneck!
But with the help of Theorem \ref{tight_bound}, we can essentially construct
independent convolved outputs, called $\left\{\widehat{\mathbf{h}_i}\right\}$, that bound the covariance of
the original convolved  outputs,  $\left\{\mathbf{h}_i\right\}$.
To apply this theorem, we need to estimate the bounding covariance matrix $\mathbf{B}$,
which can be achieved with the following simple steps (the notations are consistent with Theorem \ref{tight_bound})
\begin{compactenum}[\bfseries (a)] \item We estimate the bound on correlation coefficient $r_{max}$ \item Assign $\eta=\frac{1}{1+r_{max}}$ \item Assign $\kappa = 0$ which essentially implies $\tau_1 = \tau_2$\end{compactenum}

\noindent{\it Remark:} When computing the variance $\Sigma[{c_{\textbf{x}}, c_{\textbf{x}}}]$ in the $i^{th}$ layer, we need only $k$ upper bound of covariances $\widehat{\Sigma_1}^{(i-1)},\widehat{\Sigma_2}^{(i-1)},\cdots , \widehat{\Sigma_k}^{(i-1)}$ from the $(i-1)^{th}$ layer. Moreover, using the assumption that the covariance matrices of the $(i-1)^{\text{th}}$
layers to be identical across pixels when the input perturbation is identical, we only compute $\widehat{\Sigma}^{(i-1)}$, which in turn requires
computing only one upper bound of covariance. {\it Hence, the computational cost reduces to linear in terms of the
  depth of the network}. 
  
\noindent{\it Ansatz:}  The assumption of identical pixels (when removing the mean) is sensible when the network is linear. 
But the assumption is undesirable. 
So, we will need a mechanism to deal with the nonlinear activation function setting. 
Further, we will need to design 
the mechanics of how to track the mean and covariance for different type of layers. We will describe the details next.
%{\it What we have and what is next:}
%%We described
%%how to calculate robustness of a network by propagating upper bounds on the perturbation distributions.
%We introduced a way to remove the dependency between pixels, so that tracking of the
%first two moments in our setting remains computationally feasible and, in fact, becomes
%linear in the depth of the network. In the next section, we explore how the covariance matrix
%for each pixel propagates to various layers of a convolution neural network. 

\subsection{Robust Training By Propagating Covariances}
\label{sec:train_layers}
% \begin{figure}[t]
% %\vspace{-1em}
% %    \setlength{\abovecaptionskip}{-0.1cm}
% %    \setlength{\belowcaptionskip}{-0.2cm} 
%         \centering
%               \includegraphics[width=0.85\columnwidth]{tracking.png}
%               \caption{\footnotesize The LeNet with tracking the bounding box or the covariance matrices over each layer. The covariance matrices are denoted as the ovals. Since bounding boxes are proportional to $||W||_1$, while covariance matrices are proportional to $||W||_2$, the covariance-based upper bound will be tighter than the box-base one.}\label{interval_and_cov}
% %               \vspace*{-2em}
% \end{figure}

\begin{figure*}[t]
%\vspace{-1em}
%    \setlength{\abovecaptionskip}{-0.1cm}
%    \setlength{\belowcaptionskip}{-0.2cm} 
        \centering
               \includegraphics[height=0.38\columnwidth]{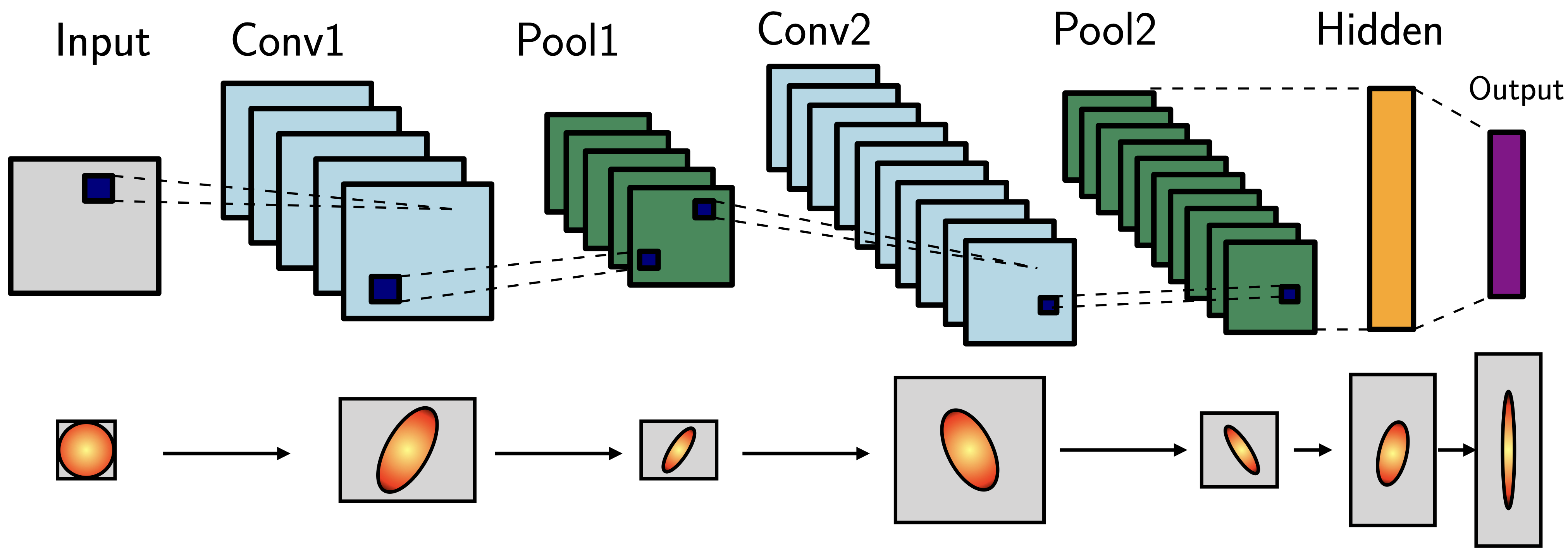}
               \includegraphics[height=0.2\columnwidth]{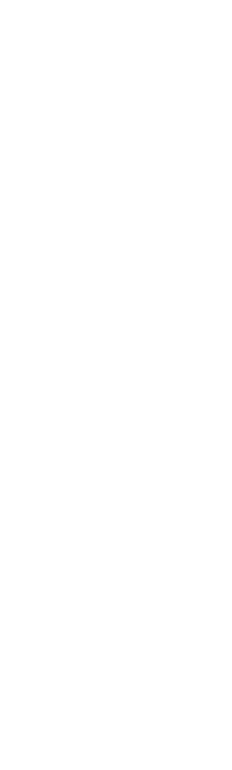}
               \includegraphics[height=0.38\columnwidth]{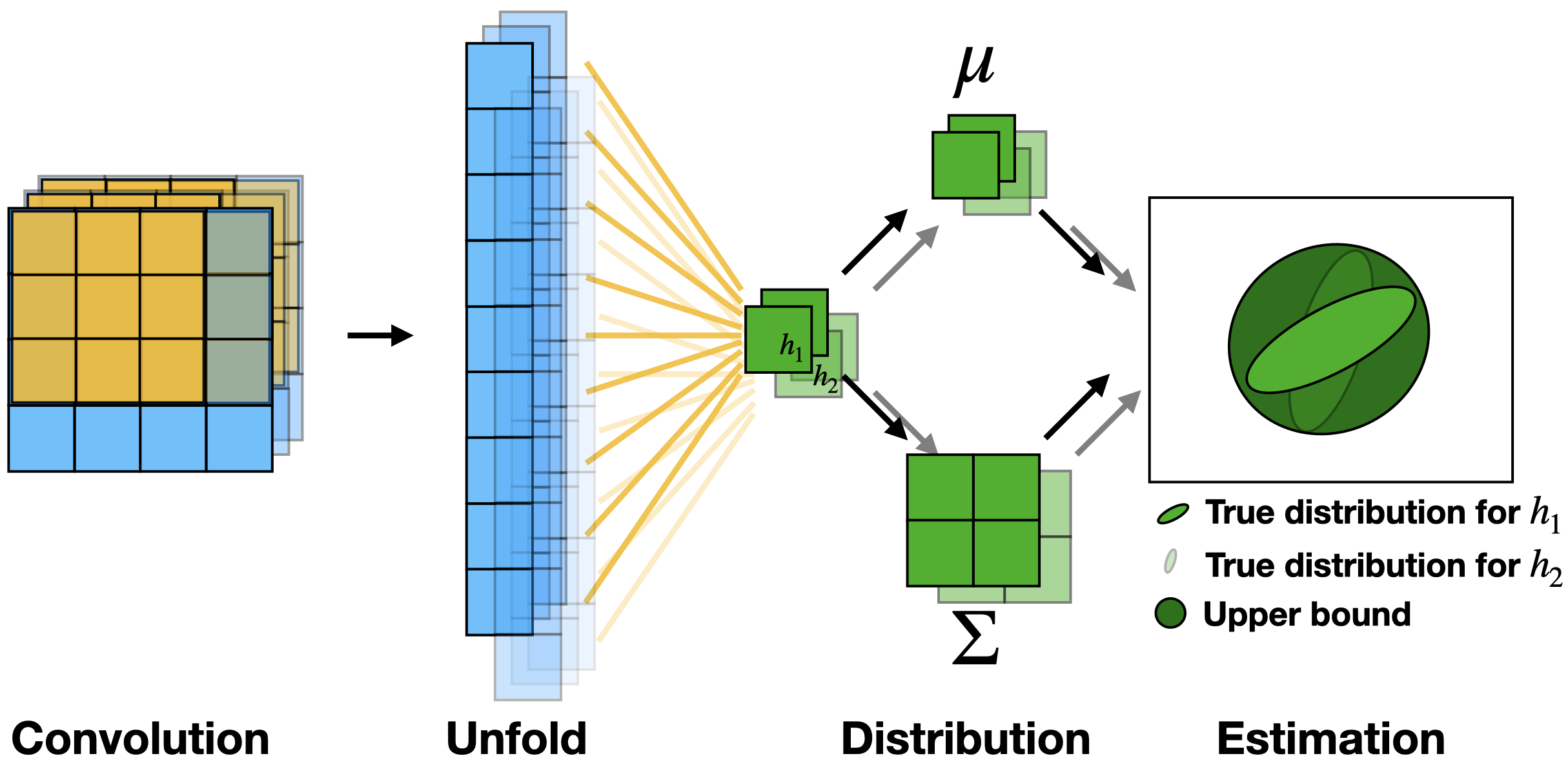}
               \caption{\footnotesize \textbf{(a) Left:} The LeNet with tracking the bounding box or the covariance matrices over each layer. The covariance matrices are denoted as the ovals. Since bounding boxes are proportional to $||W||_1$, while covariance matrices are proportional to $||W||_2$, the covariance-based upper bound will be tighter than the box-base one. \textbf{(b) Right: } The yellow blocks are the kernel of convolution, while the blue blocks are the data. After computing the distribution, we use an upper bound to remove the dependency of two pixels $h_1,h_2$. }\label{interval_and_cov}
%               \vspace*{-2em}
\end{figure*}

\noindent{\bf Overview:} We described simplifying the computation cost by tracking the upper bounds on the perturbation of
the independent pixels. 
We introduce details of an  efficient technique to track the
covariance of the distribution across different types of layers in a CNN. 
We will also describe how to deal with nonlinear activation functions.

We treat the $i^{th}$ pixel, after perturbation, as drawn from a Gaussian distribution
$\mathbf{x}_{i} \sim \mathcal{N}(\boldsymbol{\mu}_{i}, \Sigma)$, where $\boldsymbol{\mu}_{i}\in \mathbf{R}^N$
and $\Sigma$ is the covariance  matrix across the channels (note that $\Sigma$ is the same across pixels for the same layer).
We may remove the indices to simplify the formulation and 
avoid clutter.
A schematic showing propagating the
distribution across LeNet \cite{bengio2007scaling} model, for simplicity, is shown in Fig. \ref{interval_and_cov}(a) denoted
by the colored ovals.

To propagate the distribution through the whole network, we need a way to propagate the moments through the layers, 
including commonly used network modules, such as convolution and fully connected layers. 
Since the batch normalization layer normally has a large Lipschitz constant, we do not include the batch normalization layer in the network.
We will introduce the high-level idea, while the low-level details are in the appendix.

\noindent{\bf{Convolution layer:} }
Since the convolution layer is a linear operator, the covariance of an output pixel $\Sigma_{\mathbf{h}} \in \mathbf{R}^{N_{out}\times N_{out}}$ is defined as 
$\Sigma_{\mathbf{h}}=W^T\widetilde{\Sigma} W$. Here, let $\widetilde{\mathbf{x}} \in \mathbf{R}^{N_{in}k^2}$ be the vector consisting of all the independent variables inside a $k\times k$ kernel $\{\mathbf{x}_i\}$, $\widetilde{\Sigma} \in \mathbf{R}^{N_{in}k^2\times N_{in}k^2}$ is the covariance of the concatenated $\widetilde{\mathbf{x}}$. $W$ is the reshaped weight matrix of the shape $N_{in}k^2\times N_{out}$.

% \begin{figure}[bt]
% % \vspace{-2em}
%     \setlength{\abovecaptionskip}{-0.0cm}
%     \setlength{\belowcaptionskip}{-0.1cm} 
% \centering
% \includegraphics[width=0.73\columnwidth]{Conv_structure.png}
% %\vspace{-0.8em}
%               \caption{\footnotesize The yellow blocks are the kernel of convolution, while the blue blocks are the data. After computing the distribution, we use an upper bound to remove the dependency of two pixels $h_1,h_2$.}\label{conv}
%  \vspace{-1em}
% \end{figure}

We need to apply Theorem \ref{tight_bound} to compute the upper bound of $\Sigma_{\mathbf{h}}$ 
as $\widehat{\Sigma_{\mathbf{h}}}=(1+r_{max})W^T\widetilde{\Sigma} W$ to avoid the computational costs of the dependency from cross-correlations.
A pictorial description of propagating moments through the convolution layer is shown in Fig. \ref{interval_and_cov}(b).

\noindent{\bf{First (and other) linear layers:} }
The first linear layer can be viewed as a special case of convolution with kernel size equal to the input spatial dimension. Since there will only be one output neuron $\mathbf{h}$ (with channels), there is no need to break the cross-correlation between neurons.
Thus, $\Sigma_{\mathbf{h}}=W^T\Sigma_{\mathbf{\widetilde{x}}} W$ and takes a form similar to the convolution layer.

\noindent{\it Special case:} From Obs. \ref{Gaussian_estimate}, we only need the largest two intensities to estimate the $\underline{p_{c_{\mathbf{x}}}}$ in the $u_\theta(\mathbf{x})$ layer. Thus, if there is only one linear layer as the last layer in the $u_\theta(\mathbf{x})$, 
as in most of ResNet like models, this can be further simplified. We only need to consider the covariance matrix between $c_{\mathbf{x}}$ and $\widetilde{c}$ index of $u_\theta(\mathbf{x})$. Thus, this will need calculating a $2\times 2$ covariance matrix instead of a $C \times C$ matrix.

On the other hand, if the network consists of multiple linear layers,
calculating the moments of the subsequent linear layers must be handled differently.
Let $\mathbf{x} \sim \mathcal{N}\left(\boldsymbol{\mu}^{(i)}_{\mathbf{x}}, \Sigma^{(i)}_{\mathbf{x}}\right) \in \mathbf{R}^{N_i}$ be the input of the $i^{th}$ linear layer given by 
$\mathbf{h} = W_i^T\mathbf{x} + \mathbf{b}_i$, then $$\mathbf{h}  \sim \mathcal{N}\left(W_i^T\boldsymbol{\mu}^{(i)}_{\mathbf{x}}+\mathbf{b}_i, W_i^T\Sigma^{(i)}_{\mathbf{x}}W_i\right).$$ Here, $W_i \in \mathbf{R}^{N_{i}\times N_{i+1}}$, $\mathbf{b}_i \in \mathbf{R}^{N_{i+1}}$, and $\mathbf{h} \in \mathbf{R}^{N_{i+1}}$. 

\noindent{\bf{Pooling layer:} }
Recall that the input of a max pooling layer is $\{\mathbf{x}_i\}$ where each $\mathbf{x}_i\in \mathbf{R}^{N_{in}}$ and the index $i$ varies over the spatial dimension. Observe that as we identify each $\mathbf{x}_i$ by the respective distribution $\mathcal{N}\left(\boldsymbol{\mu}_{i}, \Sigma\right)$, applying max pooling over $\mathbf{x}_i$ essentially requires computing the maximum over $\left\{\mathcal{N}\left(\boldsymbol{\mu}_{i}, \Sigma\right)\right\}$, which is not a well-defined operation. Thus, we restrict ourselves to average pooling. 
This can be viewed as a special case of the convolution layer with no overlapping and the fixed kernel: $\mathbf{h}\sim\mathcal{N}\left(\frac{1}{k^2} \sum_{\mathbf{x}_i\in \mathbb{W}} \boldsymbol{\mu}_{i}, \frac{\Sigma}{k^2} \right)$, $\mathbb{W}$ is the kernel window.

\noindent{\bf{Normalization layer:} }
For the normalization layer, given by $\mathbf{h} = (\mathbf{x} - \mu')/\sigma'$, where $\mu',\sigma'$ can be computed in different ways \cite{ioffe2015batch,ba2016layer,ulyanov2016instance}, we have $\mathbf{h}  \sim \mathcal{N}\left( (\boldsymbol{\mu}^{(i)}_{\mathbf{x}}-\mu')/\sigma', \Sigma^{(i)}_{\mathbf{x}}/\sigma'^2 \right)$. However, as the normalization layers often have large Lipschitz constant \cite{awais2020towards}, we omit these layers in this work.

\noindent{\bf{Activation layer:} }
This is the final missing piece in efficiently tracking the moments. The overall goal is to find an identical upper bound of the second moments after the activation layer when the input vectors share identical second moments. Also, the first moments should be easier to compute, and ideally, will have a closed form. 
In \cite{bibi2018analytic,lee2019probact}, the authors introduced a scheme to compute the mean and variance after a ReLU operation. 
Since ReLU is an element-wise operation, for each element (a scalar), assume $x\sim \mathcal{N}(\mu,\sigma^2)$. After ReLU activation, the first and second moments of the output are given by:
\setlength{\abovedisplayskip}{0pt}
\setlength{\belowdisplayskip}{0pt}
\begin{align*}
    \mathbb{E}(\text{ReLU}(x))&=\frac{1}{2}\mu-\frac{1}{2}\mu\ \textsf{erf}(\frac{-\mu}{\sqrt{2}\sigma}) + \frac{1}{\sqrt{2\pi}}\sigma\exp(-\frac{\mu^2}{2\sigma^2}),\\
    \text{var}(\text{ReLU}(x))&<\text{var}(x)
\end{align*}
Here, $\textsf{erf}$ is the Error function. 
Since we want an identical upper bound of the covariance matrix after $\text{ReLU}$, as well as the closed form of the mean, 
%As, we need to track an upper bound of the covariance matrix, 
we use $\text{ReLU}(\mathbf{x}) \sim \mathcal{N}\left(\boldsymbol{\mu}_a, \Sigma_a\right)$ where, 
\setlength{\abovedisplayskip}{0pt}
\setlength{\belowdisplayskip}{0pt}
\begin{align*}
    \boldsymbol{\mu}_a &=\frac{1}{2}\boldsymbol{\mu}-\frac{1}{2}\boldsymbol{\mu}\ \textsf{erf}(\frac{-\boldsymbol{\mu}}{\sqrt{2}\boldsymbol{\sigma}}) + \frac{1}{\sqrt{2\pi}}\boldsymbol{\sigma}\exp(-\frac{\boldsymbol{\mu}^2}{2\boldsymbol{\sigma}^2}), \\
    \Sigma_a & \preceq \Sigma
\end{align*}
$\boldsymbol{\sigma}$ is the square root of the diagram of $\Sigma$, $\boldsymbol{\mu}$ is the mean of the input vector. All the operators in the first equation are element-wise operators.

\noindent{\bf{Last layer/prediction:} }
The last layer is the layer before softmax layer,
which represents the ``strength'' of the model for a specific class. By Obs. \ref{Gaussian_estimate}, we have the
estimate 
$$p_{c_\mathbf{x}}=\underline{p_{c_\mathbf{x}}}=\Phi\left(\frac{\boldsymbol{\mu}[{c_\mathbf{x}}]-\boldsymbol{\mu}[{\widetilde{c}}]}{\sqrt{\Sigma[{{c_\mathbf{x}},{c_\mathbf{x}}}]+\Sigma[{\widetilde{c},\widetilde{c}}]-2\Sigma[{{c_\mathbf{x}},\widetilde{c}}]}}\right)$$ 
and $p_{\widetilde{c}}=\overline{p_{\widetilde{c}}}=1-p_{c_\mathbf{x}}$ as an upper bound estimation. 
By Theorem \ref{certified_radius}, the certified radius is 

\begin{align}
C_R =& \frac{\sigma}{2}(\Phi^{-1}(\underline{p_{c_\mathbf{x}}})-\Phi^{-1}(\overline{p_{\widetilde{c}}}))\\
=& \sigma \frac{\boldsymbol{\mu}[c]-\boldsymbol{\mu}[{\widetilde{c}}]}{\sqrt{\Sigma[{{c_\mathbf{x}},{c_\mathbf{x}}}]+\Sigma[{\widetilde{c},\widetilde{c}}]-2\Sigma[{{c_\mathbf{x}},\widetilde{c}}]}}.
\end{align}
\begin{figure}[tb]
% \vspace*{-5em}
\centering
\includegraphics[width=0.60\columnwidth]{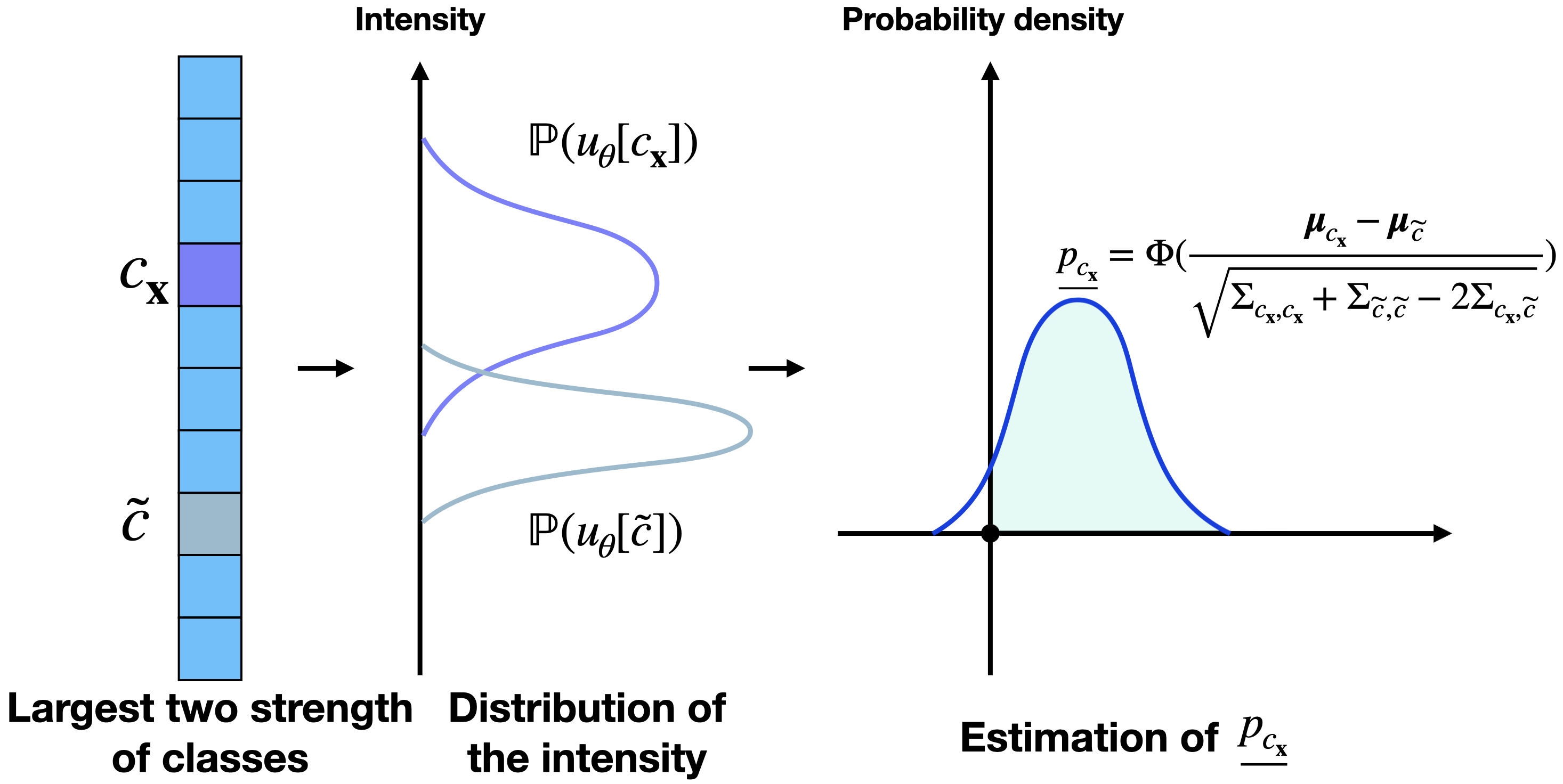}
            %   \vspace*{-2em}
               \caption{\footnotesize In the last layer, we first find the indexes of the largest two intensity ${c_\mathbf{x}},\widetilde{c}$. Then compute the $\underline{p_{c_\mathbf{x}}}$. }\label{lastlayer}\vspace*{-1em}
\end{figure}

{
\noindent{\bf {Network structures used:}}
In the experiment, we applied two types of network on different dataset, LeNet \cite{bengio2007scaling} and PreActResnet-18 \cite{he2016identity}. 

LeNet requires convolution layer, average pooling layer, activation layer, and linear layer. We build the network with three convolution layers with activation and pooling after each layer, and two linear layers. 

The structure of PreActResnet-18 is similar with two major differences -- the residual connection and it involves only one linear layer. For the residual connection, it can be viewed as a special type of linear layer. Due 
to the assumption of independence, the final covariance is the addition of two inputs. Also, there is only one linear layer as the final layer. Thus, we can reduce the cost of computing the whole covariance matrix to only computing the covariance matrix of the largest two intensities. 

As discussed above, we removed all batch normalization layers within the network as well as replaced all max pooling operations to the average pooling layer in the network structure.
Our experiments suggest that there 
is minimal impact on performance. 

}

\begin{table}[tb]
\caption {\footnotesize A review of different layers. Here, $\boldsymbol{\mu}_i,\Sigma_i$ is the mean and covariance matrix of the input channels, while $\boldsymbol{\mu}_o,\Sigma_o$ is the mean and covariance after that layer. } \label{tab:layer} 
\centering
\scalebox{0.73}{
\begin{tabular}{c|c|c|c|c}
\topline\myrowcolour
 & Convolution & Linear & Pooling & Activation\\
\hline
$\boldsymbol{\mu}_o$ & $conv(\boldsymbol{\mu}_i,W)+b$ & $W^T\boldsymbol{\mu}_i+b$ & $\frac{1}{k^2}\sum \boldsymbol{\mu}_i$ & $\frac{1}{2}\boldsymbol{\mu}_i-\frac{1}{2}\boldsymbol{\mu}_i\ \textsf{erf}(\frac{-\boldsymbol{\mu}_i}{\sqrt{2}\boldsymbol{\sigma}})$ 
\\ & & & & $+ \frac{1}{\sqrt{2\pi} \boldsymbol{\sigma}}\exp(-\frac{\boldsymbol{\mu}_i^2}{2\boldsymbol{\sigma}^2})$\\
\hline
$\Sigma_o$ & $(1+r_{max}) W^T \widetilde{\Sigma_i} W$ & $W^T {\Sigma_i} W$ & $\frac{1}{k^2}\Sigma_i$ & $\Sigma_i$\\
\bottomline 
\end{tabular}
}
\vspace*{-2em}
\end{table}

\subsection{Training Loss}
%Now that we have defined the basic modules and the ways to propagate the distribution through these modules, we are at the juncture of make the entire model trainable. In order to do that, we need to define an appropriate training loss which is the purpose of this section. 
In the spirit of  \cite{zhai2020macer}, the training loss consists of two parts: the classification loss and the robustness loss, i.e., the total loss  $l(g_\theta;\mathbf{x},y)=l_C(g_\theta;\mathbf{x},y)+\lambda l_{C_R}(g_\theta;\mathbf{x},y)$. 
%
%Let the distribution of the output of the last layer be $u_\theta(\mathbf{x})\sim\mathcal{N}(\boldsymbol{\mu},\Sigma)$. We need to output the expectation as the prediction, i.e.,  $\boldsymbol{\mu} = \mathbb{E}[u_\theta(\mathbf{x})]$. 
Similar to the literature, we use the softmax layer on the expectation to compute the cross-entropy of the prediction and the true label, given by $$
    l_C(g_\theta;\mathbf{x},y) = y \log(\text{softmax}(\mathbb{E}[u_\theta(\mathbf{x})]))
    $$
Here, $l_{C_R}(g_\theta;\mathbf{x},y=c_{\mathbf{x}})$ is 
{
$$
\max(0, \Gamma - \sigma \frac{\boldsymbol{\mu}[c_{\mathbf{x}}]-\boldsymbol{\mu}[{\widetilde{c}}]}{\sqrt{\Sigma[{c_{\mathbf{x}},c_{\mathbf{x}}}]+\Sigma[{\widetilde{c},\widetilde{c}}]-2\Sigma[{c_{\mathbf{x}},\widetilde{c}}]}})
$$
}
{Thus, minimizing the loss of $l_{C_R}$ is equivalent to maximizing $C_R$.} $\Gamma$ is the offset to control the certified radius.

\section{Experiments}
In this section, we discuss the applicability and usefulness of our proposed model in two applications namely \begin{inparaenum}[\bfseries (a)] \item image classification tasks to show the performance of our proposed model both in terms of performance and speed \item trainability of our model on data with noisy labels.\end{inparaenum} 

\subsection{Robust Training}

%One of the metric of interest in adversarially robust classification is the certified test accuracy at radius $r$, which is $\frac{1}{m}\sum_{i=1}^{m}\mathbbm{1}_{g_\theta(\mathbf{x}^i+\boldsymbol{\delta})=y^i},\forall ||\boldsymbol\delta||_2<r$. 
%However, it is difficult to compute this metric when $g_\theta(\cdot)$ is from randomized smoothing. 
Similar to \cite{cohen2019certified}, we use the approximate certified test set accuracy as our metric,
which is defined as the percentage of test set whose $C_R\leq r$. 
For a fair comparison, we use the Monte Carlo method introduced in \cite{cohen2019certified} Section 3.2 to compute $C_R$ here just as our baseline model does. Recall that $C_R = 0$ if the classification is wrong. Otherwise,  $C_R=\sigma \Phi^{-1}(\underline{p_A})$ (please refer to the  pseudocode in \cite{cohen2019certified}). In order to run certification, we used the code provided by \cite{cohen2019certified}. We also report the average certified radius (ACR), which is defined as $\frac{1}{m}\sum_{i=1}^{m}C_R(\mathbf{x}^i)$ over the test set.
\begin{figure}[bt]
% \vspace*{-1em}
%    \setlength{\abovecaptionskip}{-0.1cm}
%    \setlength{\belowcaptionskip}{-0.1cm} 
\centering
               \includegraphics[width=0.47\columnwidth]{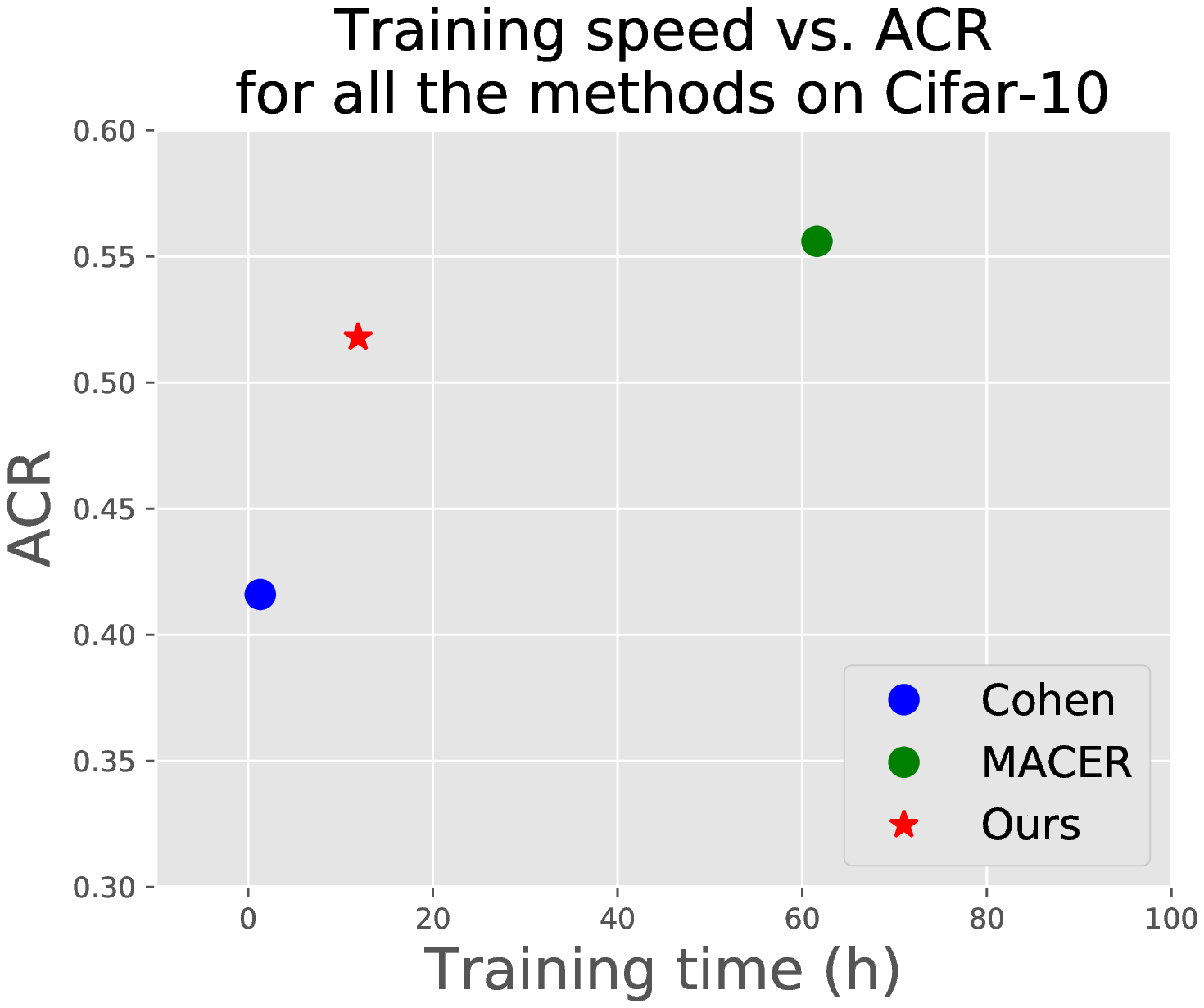}
           		\includegraphics[width=0.47\columnwidth]{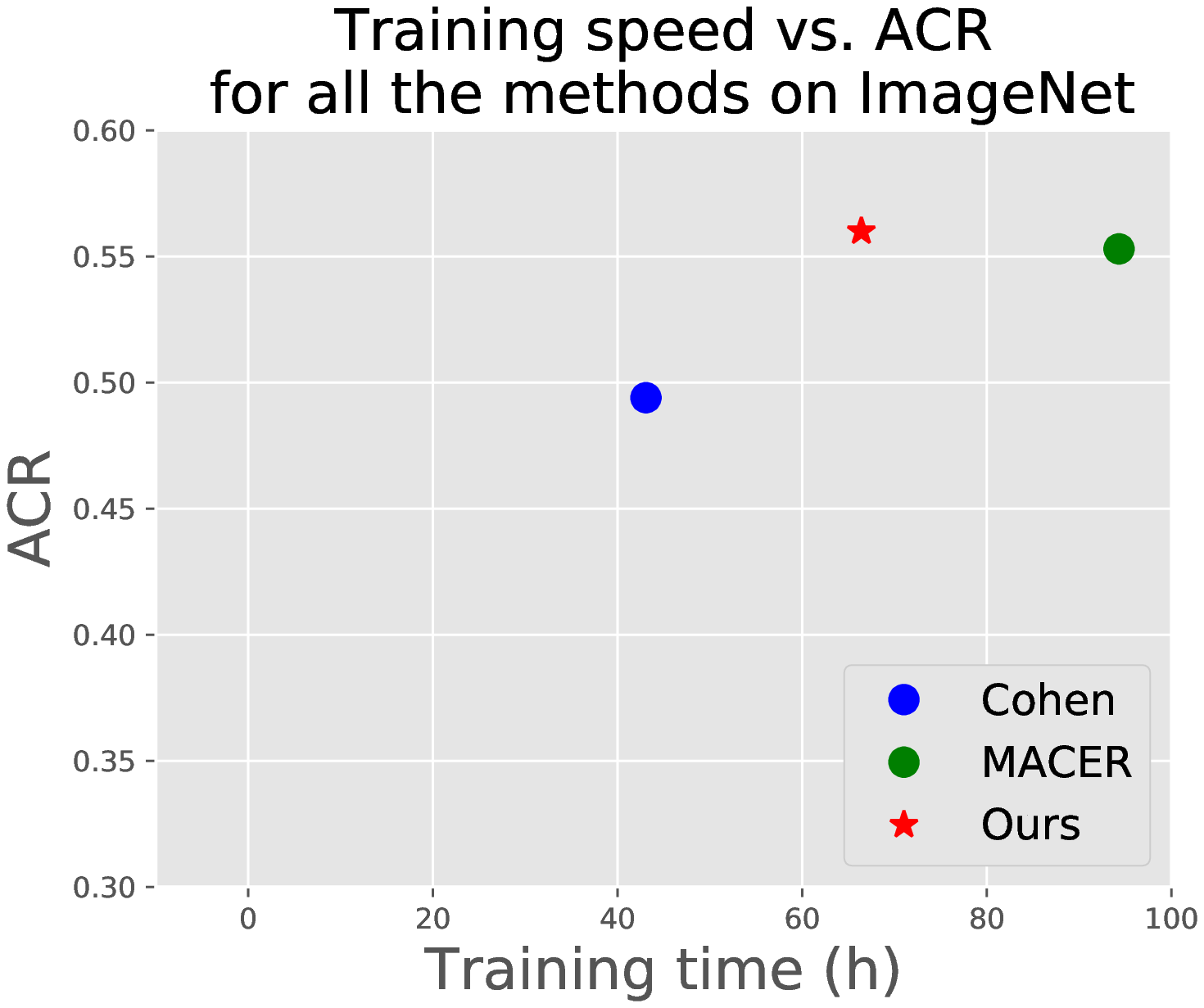}
%               \vspace*{-1em}
               \caption{\footnotesize The training speed for three models on Cifar-10 and ImageNet dataset, including  \cite{cohen2019certified}, MACER \cite{zhai2020macer}, and ours. }
               \label{speed}
\vspace*{-1.5em}
\end{figure}\\
%\vspace*{-0.5em}
{\bf{Datasets and baselines:}} We evaluate our proposed model on five vision datasets: MNIST \cite{lecun1998mnist}, SVHN \cite{netzer2011reading}, Cifar-10 \cite{krizhevsky2009learning}, ImageNet \cite{deng2009imagenet}, and Places365 \cite{zhou2017places}. 
We modify LeNet for MNIST dataset and PreActResnet-18 \cite{he2016identity} for SVHN, Cifar-10, ImageNet, and Places365 datasets similarly as in \cite{cohen2019certified}. Our baseline model is based on Monte Carlo samples, which requires a large number of samples to make an accurate estimation. {\it In the rest of the section, we will observe that our model can be at best $5\times$ faster than the baseline model. For the larger dataset, since MACER \cite{zhai2020macer} uses a reduced number of MC samples, our model is $1.5\times $ faster. }\\
{\bf{Model hyperparameters:} }
During training, we use a similar strategy as our baseline model. We train the base classifier first and then fine-tune our model considering the aforementioned robust error. We train a total of $200$ epochs with the initial learning rate to be $0.01$ for MNIST, SVHN, and Cifar-10. The learning rate decays at $100,150$ epochs respectively. The $\lambda$ is set to be $0$ in the initial training step, and changes to $4.0$ at epoch $100$ for MNIST, SVHN, and Cifar-10. 
For ImageNet and Places365 dataset, we train 120 epochs with $\lambda$ being 0.5 after 30 epochs. The initial learning rate is $0.01$ and decays linearly at $30,60,90$ epochs.\\
{\bf{Results:}}
We report the numerical results in Table. \ref{tab:robust_tab}, where, 
the number reported in each column represents the ratio of the test set with the certified radius larger than the header of that column. Thus, the larger the number is, the better the performance of different models. The ACR is the average of all the certified radius on the test set. Note that the certified radius is $0$ when the classification result is wrong.
It is noticeable that our model strikes a balance of robustness and the training speed. We achieve  $5\times$ speed-up over \cite{zhai2020macer}, which uses the Monte Carlo method during the training phase, as shown in Fig. \ref{speed}. On the other hand, compared with \cite{cohen2019certified}, our model achieves competitive accuracy and certified radius. We like to point out that although SmoothAdv \cite{salman2019provably} is a powerful model, we did not compare with SmoothAdv because MACER  \cite{zhai2020macer} performs better than SmoothAdv \cite{salman2019provably} in terms of ACR and training speed.

\begin{table}[t]
 \caption {\footnotesize Statistics for different layers of MC sampling and our upper bound tracking method. } \label{tab:stat} 
 \centering
 \scalebox{0.65}{
 \begin{tabular}{c|ccccc}
 \topline\myrowcolour
 Layer number & 1 & 5 & 9 & 13 & 17\\
 \hline
 MC (1000 samples) 	& 0.243 & 0.913 & 2.740  & 2.999 & 0.712 \\
 Upper bound 		& 0.256 & 1.126 & 4.069  & 5.367 & 1.208 \\
\bottomline 
 \end{tabular}
 }
 \vspace*{-1em}
 \end{table}

{
Separate from the quantitative performance measures, we also 
evaluate the validity of Gaussian assumption on the pre-activation vectors within the network. Here, we choose PreActResNet-18 on ImageNet to visualize the first two channels across different layers. The detailed results are shown in Fig. \ref{first_assumption} and Table. \ref{tab:stat}. These results not only show that the assumption is reasonable along the neural network, but also demonstrates that our method can estimate the covariance matrix well when the depth of network is moderate. Deeper networks, where the bounds get looser, are described in the appendix.
%We will discuss about our method in the %following section.
}

\begin{figure}[!b]
%\vspace*{-1em}
%    \setlength{\abovecaptionskip}{-0.0cm}
%    \setlength{\belowcaptionskip}{-0.0cm} 
        \centering
               \includegraphics[height=0.30\columnwidth]{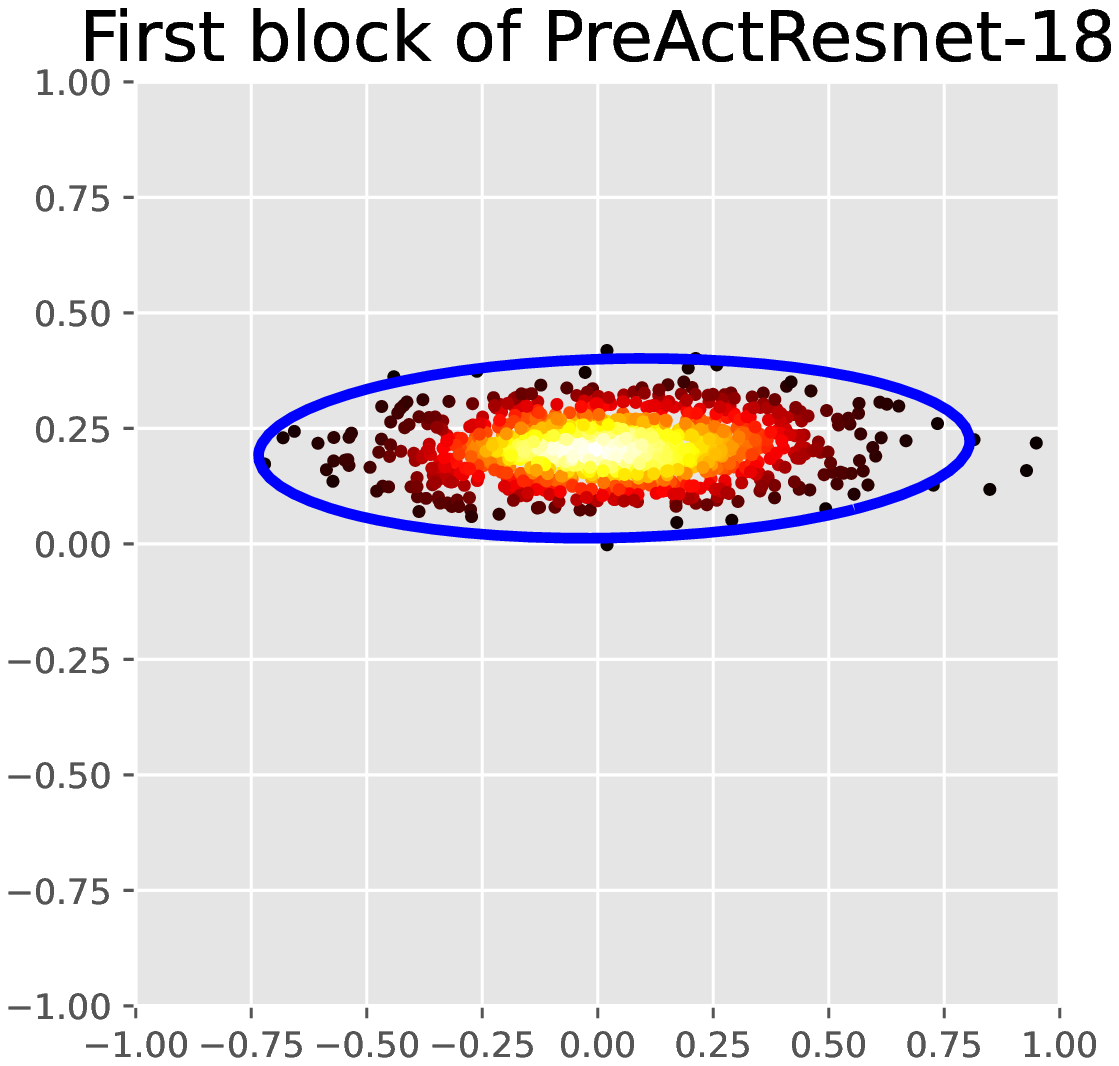}
               \includegraphics[height=0.30\columnwidth]{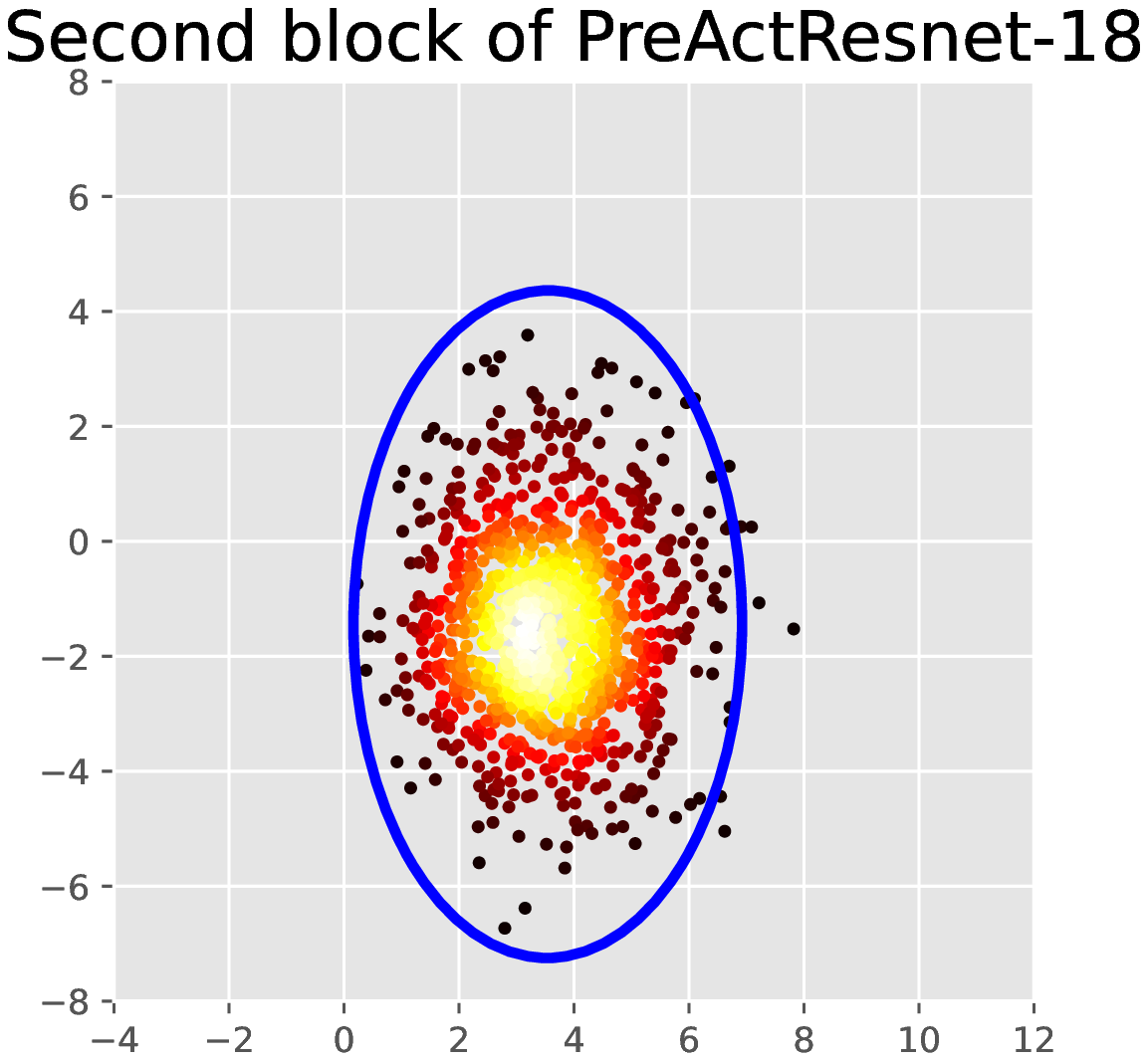}
               \includegraphics[height=0.30\columnwidth]{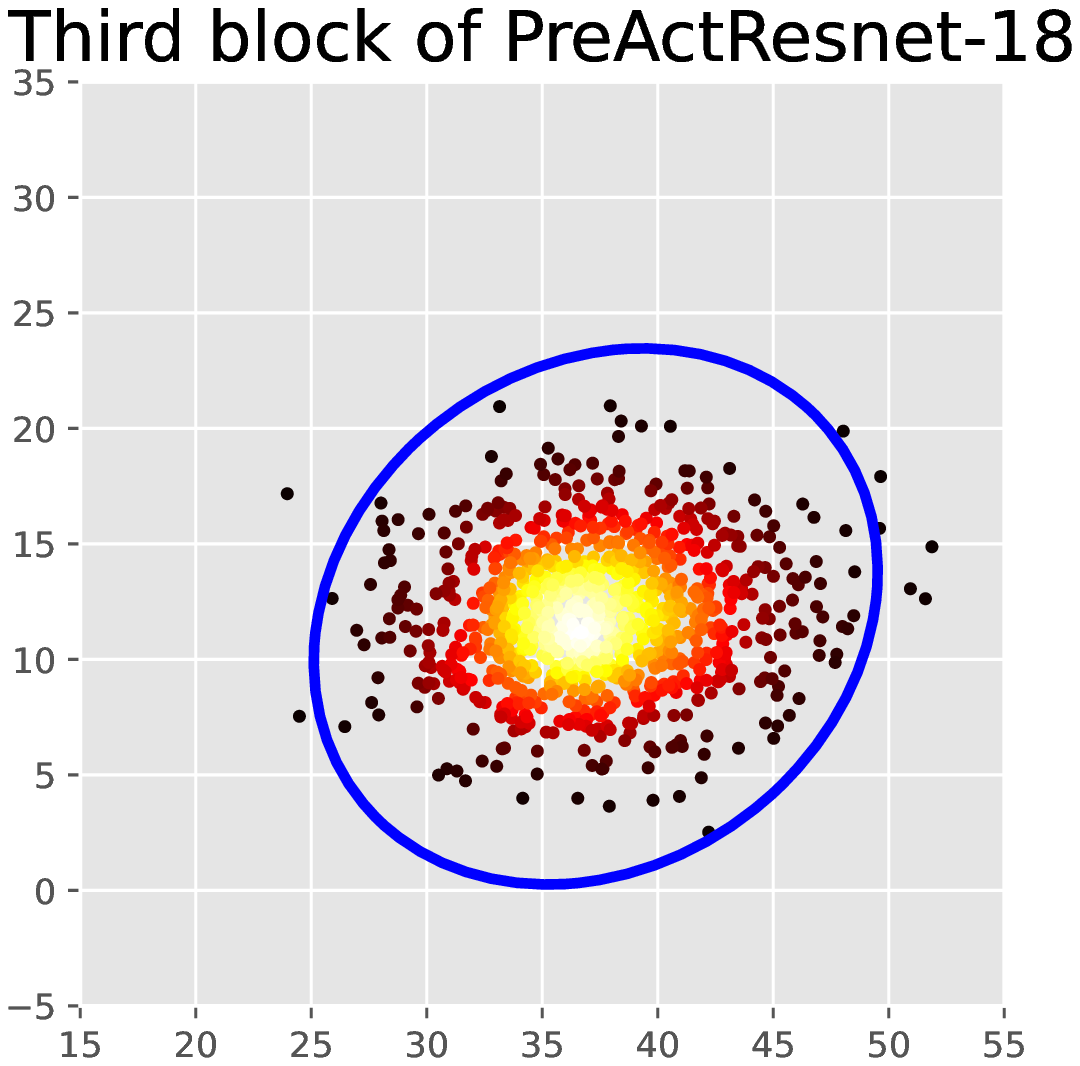}
               \caption{\footnotesize  A visualization of the first two channels within the neural network across different layers. The dots are the actual MC samples and the color represents the density at that point. The blue oval is generated from the covariance matrices we are tracking.}
               \label{first_assumption}
%                 \vspace*{-1em}
\end{figure}

\noindent{\bf{Ablation study:}}
We perform an ablation study on the choice of the hyperparameters for Places365 dataset. We fix $\sigma$ of the perturbation to be $0.5$. We first test the influence of $\lambda$ which is the balance between the accuracy (first moments) and the robustness (second moments). Also, to verify the estimation of  $r_{max}$, we tried different $r_{max}$ estimates while fixing $\lambda=0.5$. The detailed results are shown in Table. \ref{tab:ablation}. 

\noindent{\bf{Discussion:}}
A key benefit of our method is the training time. As shown in Fig. \ref{speed}, our method can be $5\times $ faster on Cifar-10, dataset, with a comparable ACR as MACER. For larger datasets, since MACER reduces the number of MC samples in their algorithm, our method is only $1.5\times$ faster with a slightly better ACR than MACER. {\it Hence, our method is a cheaper substitute of the SOTA with a marginal performance compromise.} 

\newcolumntype{g}{>{\columncolor{Gray}}c}
\begin{table*}[h]
\caption {\footnotesize The results on MNIST, SVHN, Cifar-10, ImageNet, and Places365 with the certified robustness. The number reported in each column represents the ratio of the test set with the certified radius larger than the header of that column under the perturbation $\sigma$. ACR is the average certified radius of all the test samples. A larger value is better for all the numbers reported.} \label{tab:robust_tab} 
\centering
\begin{minipage}{.49\textwidth}
\centering
\scalebox{0.58}{
%\begin{threeparttable}
\begin{tabular}{c|c|c|c g c g c g c g |c}
\topline\myrowcolour
Dataset&$\sigma$ & Method & 0.00 & 0.25 & 0.50 & 0.75 & 1.00 & 1.25 & 1.50 & 1.75  & ACR  \\
\hline

\multirow{6}{*}{MNIST} & \multirow{3}{*}{0.25} & Cohen \cite{cohen2019certified} & 0.99 & 0.97 & 0.94 & 0.89 & 0&0&0&0 & 0.887\\
 & & MACER \cite{zhai2020macer} & 0.99 & 0.99&  0.97&  0.95 &   0&0&0&0&\highest{0.918}\\
 & & Ours  & 0.99 & 0.98 & 0.96 & 0.92  &  0&0&0&0&0.904\\
\hhline{~-----------}
& \multirow{3}{*}{0.50} & Cohen \cite{cohen2019certified} & 0.99 & 0.97 & 0.94 & 0.91 & 0.84 & 0.75 & 0.57 & 0.33 &1.453\\
 & & MACER \cite{zhai2020macer} & 0.99 & 0.98 & 0.96 & 0.94 & 0.90 & 0.83 & 0.73 & 0.50 & \highest{1.583}\\
&  & Ours & 0.98 & 0.98 & 0.95 & 0.91 & 0.87 & 0.77 & 0.62 & 0.37  & 1.485\\
\hline
\multirow{6}{*}{SVHN} & \multirow{3}{*}{0.25} & Cohen \cite{cohen2019certified} & 0.90 & 0.70 & 0.44 & 0.26 & 0&0&0&0 & 0.469\\
& & MACER \cite{zhai2020macer} & 0.86 & 0.72&  0.56&  0.39 &   0&0&0&0&\highest{0.540}\\
& & Ours  & 0.89 & 0.68 & 0.48 & 0.36  &  0&0&0&0&0.509\\
\hhline{~-----------}
& \multirow{3}{*}{0.50} & Cohen \cite{cohen2019certified} & 0.67 & 0.48 & 0.37 & 0.24 & 0.14 & 0.08 & 0.06 & 0.03 &0.434\\
&  & MACER \cite{zhai2020macer} & 0.61 & 0.53 & 0.44 & 0.35 & 0.24 & 0.15 & 0.09 & 0.04 &\highest{0.538}\\
&  & Ours & 0.67 & 0.53 & 0.36 & 0.29 & 0.19 & 0.12 & 0.07 & 0.03  & 0.475\\
\hline
\multirow{6}{*}{Cifar-10} & \multirow{3}{*}{0.25} & Cohen \cite{cohen2019certified} & 0.75 & 0.60 & 0.43 & 0.26 & 0&0&0&0 & 0.416\\
& & MACER \cite{zhai2020macer} & 0.81 & 0.71&  0.59&  0.43 &   0&0&0&0&\highest{0.556}\\
& & Ours  & 0.80 & 0.72 & 0.55 & 0.37  &  0&0&0&0&0.518\\
\hhline{~-----------}
& \multirow{3}{*}{0.50} & Cohen \cite{cohen2019certified} & 0.65& 0.54& 0.41& 0.32& 0.23& 0.15& 0.09& 0.04 &0.491\\
& & MACER \cite{zhai2020macer} & 0.66& 0.60& 0.53& 0.46& 0.38& 0.29& 0.19& 0.12 & \highest{0.726}\\
& & Ours & 0.58 & 0.56 & 0.43 & 0.36 & 0.27 & 0.15 & 0.08 & 0.01  & 0.543\\
\bottomline  
\end{tabular}
}
\end{minipage}
\begin{minipage}{.49\textwidth}
\centering
\scalebox{0.58}{
%\begin{threeparttable}
\begin{tabular}{c|c|c|c g c g c g c g |c}
\topline\myrowcolour
Dataset&$\sigma$ & Method & 0.00 & 0.25 & 0.50 & 0.75 & 1.00 & 1.25 & 1.50 & 1.75  & ACR  \\
\hline	

\multirow{9}{*}{ImageNet} & \multirow{3}{*}{0.25} & Cohen \cite{cohen2019certified} & 0.58 & 0.49 & 0.40 &0.29&0&0&0&0 & 0.379\\
& & MACER\cite{zhai2020macer} & 0.59 & 0.52 & 0.43& 0.34&0&0&0&0&0.418\\
& & Ours  & 0.64 & 0.55 & 0.44 & 0.33 & 0&0&0&0 & \highest{0.425}\\
\hhline{~-----------}
& \multirow{3}{*}{0.50} & Cohen \cite{cohen2019certified} & 0.43 & 0.38 & 0.34 & 0.29 & 0.26 & 0.22 & 0.17 & 0.12 & 0.494\\
& & MACER \cite{zhai2020macer} & 0.54 & 0.47 & 0.39 & 0.32 & 0.29& 0.21 & 0.17 & 0.11 &0.553\\
& & Ours  & 0.52 & 0.47 & 0.39 & 0.32 & 0.28& 0.23 & 0.18 &0.13& \highest{0.560}\\
\hhline{~-----------}
 & \multirow{3}{*}{1.00} & Cohen \cite{cohen2019certified} & 0.21 & 0.19 & 0.18 & 0.16 &0.15 & 0.13 & 0.11 & 0.09 & 0.345\\
 & & MACER \cite{zhai2020macer} & 0.37 & 0.33 &  0.30 & 0.26 & 0.22 & 0.19 & 0.15 & 0.12 & 0.517\\
 & & Ours  & 0.38 & 0.33 & 0.29 & 0.26 & 0.22 & 0.19 & 0.15 & 0.11 & \highest{0.519}\\
\hline

\multirow{9}{*}{Places365} & \multirow{3}{*}{0.25} & Cohen \cite{cohen2019certified} & 0.45 & 0.42 & 0.36 & 0.29 &0&0&0&0 & 0.340\\
& & MACER\cite{zhai2020macer} & 0.46 & 0.44 & 0.39 & 0.30 &0&0&0&0& 0.359 \\
& & Ours  & 0.50 & 0.46 & 0.40 & 0.33 & 0&0&0&0& \highest{0.380} \\
\hhline{~-----------}
& \multirow{3}{*}{0.50} & Cohen \cite{cohen2019certified} & 0.43 & 0.38 & 0.35 & 0.28 & 0.23 & 0.19 & 0.17 & 0.12 & 0.484\\
& & MACER \cite{zhai2020macer} & 0.45 & 0.42 & 0.37 & 0.31 & 0.26 & 0.22 & 0.18 & 0.13 & 0.533\\
& & Ours  & 0.46 & 0.43 & 0.39 & 0.35 & 0.31 & 0.28 & 0.23 & 0.16 & \highest{0.597}\\
\hhline{~-----------}
 & \multirow{3}{*}{1.00} & Cohen \cite{cohen2019certified} & 0.20 & 0.18 & 0.16 & 0.15 & 0.13 & 0.12 & 0.11 & 0.10 & 0.357\\
 & & MACER \cite{zhai2020macer} & 0.31 & 0.29 & 0.28 & 0.25 & 0.22 & 0.21 &  0.19 & 0.17 & 0.615\\
 & & Ours  & 0.32 & 0.30 & 0.29 & 0.26 & 0.24 & 0.21 & 0.19 & 0.16 & \highest{0.622}\\
\bottomline 
\end{tabular}
}
\end{minipage}
%\vspace*{-1em}
\end{table*}

\begin{table}[!ht]
\caption {\footnotesize Ablation experiment on Places365 with $\sigma=0.5$. We perform the choice of $\lambda$ and $r_{max}$ as the hyper-parameters. } \label{tab:ablation} 
\centering
\scalebox{0.58}{
\begin{tabular}{c|c|c g c g c g c g |c}
\topline\myrowcolour
Parameters & Value & 0.00 & 0.25 & 0.50 & 0.75 & 1.00 & 1.25 & 1.50 & 1.75  & ACR\\
\hline
\multirow{3}{*}{$\lambda$} & 0.0 & 0.43 & 0.38 & 0.35 & 0.28 & 0.23 & 0.19 & 0.17 & 0.12 & 0.484\\
 & 0.5 & 0.47 & 0.44 & 0.39 & 0.34 & 0.29 & 0.23 & 0.19 & 0.14 & \highest{0.565}\\
 & 1.0 & 0.44 & 0.41 & 0.34 & 0.30 & 0.28 & 0.23 & 0.20 & 0.14 & 0.530\\
\hline
\multirow{5}{*}{$r_{max}$} & 0.0 & 0.43 & 0.38 & 0.36 & 0.31 & 0.27 & 0.21 & 0.16 & 0.13 & 0.509\\
 & 0.1 & 0.47 & 0.44 & 0.39 & 0.34 & 0.29 & 0.23 & 0.19 & 0.14 & 0.565 \\
 & 0.2 & 0.46 & 0.43 & 0.39 & 0.35 & 0.31 & 0.28 & 0.23 & 0.16 & \highest{0.597} \\
 & 0.3 & 0.46 & 0.44 & 0.41 & 0.35 & 0.29 & 0.23 & 0.19 & 0.15 & 0.573 \\
 & 0.4 & 0.44 & 0.40 & 0.36 & 0.31 & 0.27 & 0.23 & 0.17 & 0.12 & 0.520\\
\bottomline 
\end{tabular}
}
\end{table}

\begin{table}[b]
%\vspace{-2em}
\caption {\footnotesize Average test accuracy on pair-flipping with noise rate $45\%$ for last $10$ epochs. Results of Bootstrap\cite{reed2014training}, S-model\cite{goldberger2016training}, Decoupling\cite{malach2017decoupling}, MentorNet\cite{jiang2017mentornet}, Co-teaching\cite{han2018co}, Trunc $\mathcal{L}_q$\cite{zhang2018generalized}, and Ours.} \label{tab:noisy_cifar} 
\centering
\scalebox{0.58}{
\begin{tabular}{c|c g c g c g c}
\topline\myrowcolour
Method & Bootstrap & S-model & Decoupling & MentorNet & Co-teaching & Trunc $\mathcal{L}_q$ & Ours \\
\hline
mean & 0.501 & 0.482 & 0.488 & 0.581 & 0.726 & \highest{0.828} & 0.808\\
std                          &  3.0$e-3$ & 5.5$e-3$ & 0.4$e-3$ & 3.8$e-3$ & 1.5$e-3$ & 6.7$e-3$ & 0.2$e-3$ \\ 
\bottomline 
\end{tabular}
}
%\vspace*{-1em}
\end{table}

\noindent{\bf Limitations:} There are some limitations due to the 
simplifications incorporated in 
our model. When the network is extremely deep, e.g., Resnet-101, the estimation of the second moments tends to be looser as the network grows deeper. Another minor issue is when the input perturbation is large. As observed from Table. \ref{tab:robust_tab}, the ACR drops for $\sigma=1.0$ from $\sigma=0.5$. The main reason is the assumption that samples are Gaussian distributed. Hence, as the perturbation grows larger, the number of channels, by the central limit theorem, should be much larger to satisfy the Gaussian distribution. Thus, given a network, there is an inherent limitation imposed on the input perturbation. We provide a more detailed discussion in the appendix.

\subsection{Training With Noisy Labels}
{
As we discussed in Section. \ref{sec:intro}, training a robust network has a side effect on smoothing the margin of the decision boundary, which enables training with noisy labels.\\
}
\textbf{Problem statement: }Here, we consider a challenging noise setup called ``pair flipping'', which can be described as follows. When noise rate is $p$ fraction, it means $p$ fraction of the $i_{th}$ labels are flipped to the $(i+1)_{th}$. In this work, we test our method on the high noise rate $0.45$. \\
%The detailed transition matrix is $\begin{bmatrix} 0.55 & 0.45 & 0 & \ldots & 0 \\ 0 & 0.55 & 0.45 & \ldots & 0 \\ & & & \ldots & \\ 0.45 & 0 & 0 & \ldots & 0.55\end{bmatrix}$. 
%Note that when $p>0.5$, it means over half of the labels in the training data are flipped to the next class label. Thus, it is equivalent to reassigning the class label to the predecessor label with $(1-p)$ fraction of the labels flipping to the new predecessor. 
%As for an example, for three class $a,b,c$ as labels $\{1,2,3\}$, with $p=0.55$, ``pair flipping'' entailing $55\%$ of data in class $a$ into class $b$ is equivalent to flip $45\%$ of data in class $b$ to class $a$ with rearranging the labels. 
%Normally, another type of noisy labeling, which evenly assign the wrong labels across all the rest classes, is studied in the literature. Due to the fact that our model only considers the most possible two classes, our model is more suitable to the pair flipping setup. \\
\textbf{Dataset: }The dataset we considered for this analysis is Cifar-10. To generate noisy labels from the clean labels of the dataset, we stochastically changed $p$ fraction of the labels using the source code provided by \cite{han2018co}. We perform a comparative analysis of our method with Bootstrap \cite{reed2014training}, S-model \cite{goldberger2016training}, Decoupling \cite{malach2017decoupling}, MentorNet \cite{jiang2017mentornet}, Co-teaching \cite{han2018co}, and Trunc $\mathcal{L}_q$\cite{zhang2018generalized}. \\
\textbf{Model hyperparameters: } Similar to training robust network with the clean labels, we first treat the noisy labels as ``clean'' to train our model. After $60$ epochs, we remove the classification loss for the data with top $10\%$ $C_R$ to fine-tune the network. The initial learning rate is set to $0.01$ and decays at $30,60,90$ epochs, respectively. \\
\textbf{Results:} The results are shown in Table. \ref{tab:noisy_cifar}, where, it is noticeable that even under this strong label corruption, our model outperforms most baseline results as well as stays stable over different epochs. %This shows the power of our proposed distribution based robust model. 
%{
%We note the current state-of-the-art results are reported by \cite{nguyen2019self}. But at this time, the code is not publicly available. 
%}

\section{Conclusions}

Developing mechanisms that enable training certifiably robust neural networks nicely complements the rapidly evolving body of literature on adversarial training. While certification schemes, in general, have typically been limited to small sized networks, recent proposals related to randomized smoothing have led to a significant expansion of the type of models where these ideas can be used. Our proposal here takes this line of work forward and shows that bound propagation ideas together with some meaningful approximations can provide an efficient method to maximize the certified radius -- a measure of robustness of the model. We show that the strategy achieves competitive results to other baselines with faster training speed. We also investigate a potential use case for training with noisy labels where the behavior of such ideas has not been investigated, but appears to be promising. 
%% In this paper, we proposed a novel sample free method to train the verifiable robust neural network. 
%% Unlike the adversarial training which requires generating adversarial samples when training or sample-based Gaussian noise augmentation, our model can train the robust neural network faster and with less training epochs. 
%% Apart from purely training a robust neural network, we also show the capacity of our method to learn from the extremely noisy labels. 

\vspace{-5pt}
\section*{Acknowledgments}
Research supported in part by grant NIH RF1 AG059312, NIH RF1AG062336, NSF CCF \#1918211 and NSF CAREER RI \#1252725. Singh thanks Loris D'Antoni and Aws Albarghouthi for several discussions related to the work in \cite{drews2020fairness}. 

{\small
\bibliographystyle{ieee_fullname}
\bibliography{egbib}
}

\newpage
\setcounter{section}{0}
\setcounter{table}{0}
\renewcommand\thesection{\Alph{section}}
\renewcommand{\thetable}{\alph{table}}

\section{Difficulties Of Accounting For All Terms}
In the main paper, we described at a high level,
the problem of bookkeeping
all covariances $\Sigma$ and cross-correlations $E$ because the number of terms
grow very rapidly. Here, we provide the low-level details of where/why this problem arises
and strategies to address it.

%
%\begin{figure*}[bt]
%        \centering
%               \includegraphics[width=0.9\textwidth]{exp.png}
%               \caption{\label{fig:exp}\footnotesize Consider a 1-D convolution network with kernel size $k=3$. The blue dots are the nodes as well as the covariance matrices $\Sigma$ to be computed. The red arrows are the cross-correlation $E$ between two nodes. 
%               In the last layer, we will need to compute $1$ covariance matrix. In the second last layer, we will need to compute $3$ covariance matrices $\Sigma$ and $3$ cross-correlation matrices $E$. In the third last layer, we will need to compute $3^2$ covariance matrices $\Sigma$ and $\frac{1}{2}(3^2+1)3^2$. Then in the $(q+1)^{th}$ last layer, it will require $k^q$ covariance matrices $\Sigma$ and $\frac{1}{2}(1+k^q)k^q$.} 
%\end{figure*}

\subsection{A 1-D Convolution Without Overlap: Exponential Explosion}
As shown in Table. \ref{tab:example_exp}, even if we only consider
a simple case -- a 1-D convolution network with a kernel size $k=3$, we run into problems
with a naive bookkeeping approach. Let us analyze this at the last layer and trace back the number of terms
that will contribute to it.

In the last layer, we will need to calculate a single covariance matrix.
Tracing back, in the second last layer, we will need to compute $3$ covariance matrices $\Sigma$ and
$3$ cross-correlation matrices $E$. In the third last layer, we will need to compute $3^2$
covariance matrices $\Sigma$ {and $\frac{1}{2}(3^2+1)3^2$ cross-correlation matrices $E$}.
Then, in the $(q+1)^{th}$ last layer, we will require $k^q$ covariance matrices
$\Sigma$ {and $\frac{1}{2}(1+k^q)k^q=\Theta(k^q)$ cross-correlation matrices $E$}.
In this case, the computational cost increase exponentially, which is infeasible.
in case of a deep neural network.
This setup will rapidly increase the field of view of the network,
  and one would need special types of convolution, e.g., dilated convolutions, to
  address the problem \cite{bai2018convolutional}.

%\begin{figure*}[bt]
%        \centering
%               \includegraphics[width=0.9\textwidth]{overlap.png}
%               \caption{\label{fig:overlap}\footnotesize Consider a 1-D convolution network with kernel size $k=3$ with overlapping. The blue dots are the nodes as well as the covariance matrices $\Sigma$ to be computed. The red arrows are the cross-correlation $E$ between two nodes. Similarly, in the last layer, we will need to compute $1$ covariance matrix. In the second last layer, we will need to compute $3$ covariance matrices $\Sigma$ and $3$ cross-correlation matrices $E$. In the third last layer, we will need to compute $(3-1)\times 2+1$ covariance matrices $\Sigma$ and $\frac{1}{2}((3-1)\times 2+1+1)((3-1)\times 2+1)$. Then in the $(q+1)^{th}$ last layer, it will require $(k-1)q+1$ covariance matrices $\Sigma$ and $\frac{1}{2}((k-1)q+1+1)((k-1)q+1)$.} 
%\end{figure*}

\subsection{A 1-D Convolution With Overlap: Polynomial But Higher Than Standard 2-D Convolution}
The setup above did not consider the overlap between pixels as the shared kernel
moves within the same layer.
If we consider a setting with overlap as in a standard CNN,
as shown in Table. \ref{tab:example_overlap}, similarly, in the last layer, we
will need to compute a single covariance matrix.
In the second last layer, we will need to compute $3$ covariance matrices $\Sigma$ and
$3$ cross-correlation matrices $E$. In the third last layer,
we will need to compute $(3-1)\times 2+1$ covariance matrices $\Sigma$ {and
  $\frac{1}{2}((3-1)\times 2+1+1)((3-1)\times 2+1)$ cross-correlation matrices $E$}.
Then, in the $(q+1)^{th}$ last layer, it will require $(k-1)q+1$
covariance matrices $\Sigma$ {and $\frac{1}{2}((k-1)q+1+1)((k-1)q+1)=\Theta(k^2q^2)$
 cross-correlation matrices $E$}. 

This may appear
feasible since the computational cost is polynomial.
%However, a
%normal feed-forward 1-D convolution network will cost $\Theta(kq)$.
But this setup indeed increases the computational cost of a 1-D convolution network $\Theta(k^2q^2)$ to
be higher than a 2-D convolution network, {$\Theta(k^2 q)$ since $q<q^2$}. 
When this strategy is applied to the 2-D image case, the
computational cost ($\Theta(k^4 q^2)$) will be more than the
4-D tensor convolution network ($\Theta(k^4 q)$), which is not feasible.
As shown in the literature, training CNNs efficiently
on very high resolution 3-D images (or videos) is still an active topic of ongoing research. 

\subsection{Memory Cost}
The GPU memory footprint also turns out to be high.
For a direct impression of the numbers, we take the ImageNet with PreActResnet-18 as an example,
shown in Table. \ref{tab:memory}. The memory cost for a naive method is too high, even for large clusters,
if we want to track 
%to be feasible even for the large cluster as it would require to compute all the
all cross-correlation terms between any two pixels and the channels.

\begin{table*}[ht]
\caption{\footnotesize The memory cost of ImageNet with PreActResnet-18 for different layers. The brute force method would require to compute all the cross-correlation between different pixels and channels. The memory cost for the brute force method is too high to afford. Our method, in comparison, only increase a little from the tradition (non-robust) network. }
\label{tab:memory}
\centering
\scalebox{0.79}{
\begin{tabular}{c|c|c|c}
\topline\myrowcolour
	Layer & Traditional network & Brute force method & Our method \\
	\hline
	Input ($224 \times 224 \times 3$) & $224 \times 224 \times 3=150528$ & $\frac{1}{2}(224 \times 224 \times 3)^2=11329339392$ & $(224 \times 224 \times 3)+ (3\times 3)=150537$\\
	\hline
	$1^{st}$ convolution ($112\times 112\times 64$) & $112\times 112\times 64=802816$ & $\frac{1}{2}(112\times 112\times 64)^2=644513529856$ & $(112\times 112\times 64) + (64\times 64)=806912$\\
	\hline
	$1^{st}$ res-block ($56\times 56\times 64$) & $56\times 56\times 64 = 200704$ & $\frac{1}{2}(56\times 56\times 64)^2 = 322256764928$ & $(56\times 56\times 64) + (64\times 64) = 204800$ \\
	\hline
	$2^{nd}$ res-block ($28\times 28\times 128$) & $28\times 28\times 128=100352$ & $\frac{1}{2}(28\times 28\times 128)^2=5035261952$ & $(28\times 28\times 128)+(128\times 128)=116736$ \\
	\hline
	$3^{rd}$ res-block ($14\times 14\times 256$) & $14\times 14\times 256=50176$ & $\frac{1}{2}(14\times 14\times 256)^2=1258815488$ & $(14\times 14\times 256)+(256\times 256)=115712$\\
	\hline
	$4^{th}$ res-block ($7\times 7\times 512$) & $7\times 7\times 512 = 25088$ & $\frac{1}{2}(7\times 7\times 512)^2 = 314703872$ & $(7\times 7\times 512) + (512\times 512) = 287232$\\
	\bottomline
	
\end{tabular}
}
	
\end{table*}

\section{Tracking Distributions Through Layers}
In the main paper, we briefly described the different layers used in our model. Here, we will provide
more details about the layers.

We consider the $i^{th}$ pixel, after perturbation, to be drawn from a Gaussian distribution
$\mathbf{x}_{i} \sim \mathcal{N}(\boldsymbol{\mu}_{i}, \Sigma)$, where $\boldsymbol{\mu}_{i}\in \mathbf{R}^N$
and $\Sigma$ is the covariance  matrix across the $N$ channels (note that $\Sigma$ is same for a
layer across all pixels).
In the following sections, we will remove the indices to simplify the formulation.

Several commonly-used basic blocks such as
convolution and fully connected layers are used to setup our network architecture.
In order to propagate the distribution through
the entire network, we need a way to propagate the moments through these layers. 

\subsection{Convolution Layer:}
Let the pixels $\{\mathbf{x}_i\}$ inside a $k\times k$ convolution kernel window
be independent. Let $\widetilde{\mathbf{x}} \in \mathbf{R}^{N_{in}k^2}$ be the vector
which consists of $\{\mathbf{x}_i\}$, where $N_{in}$ is the number of input channels. Let $\Sigma \in \mathbf{R}^{N_{in}\times N_{in}}$ be the covariance matrix of each pixel within the $k\times k$ window. Let $W$ be the weight matrix of the convolution layer, which is of the shape $N_{in}\times k\times k\times N_{out}$. With a slight abuse of notation, let $W$ be the reshaped weight matrix of the shape $N_{in}k^2\times N_{out}$. Further, concatenate $\Sigma$ from each $\{\mathbf{x}_i\}$ inside the $k\times k$ window to get a block diagonal $\widetilde{\Sigma} \in \mathbf{R}^{N_{in}k^2\times N_{in}k^2}$. Then, we get the covariance of an output pixel to be $\Sigma_{\mathbf{h}} \in \mathbf{R}^{N_{out}\times N_{out}}$ defined as $$\Sigma_{\mathbf{h}}=W^T\widetilde{\Sigma} W$$

We need to apply Theorem 2 (from the main paper) to compute the upper bound of $\Sigma_{\mathbf{h}}$.
We get the upper bound covariance matrix (covariance matrix of independent random variable $\widehat{\mathbf{x}}$ used in Theorem 2) as $$\widehat{\Sigma_{\mathbf{h}}}=(1+r_{\max})W^T\widetilde{\Sigma} W$$

{\it Summary:} Given each input pixel, $\mathbf{x}_{i} \in \mathbf{R}^{N_{in}}$ following $\mathcal{N}(\boldsymbol{\mu}_{i}, \Sigma)$ and convolution kernel matrix $W$, the output distribution of each pixel is $\mathcal{N}(\boldsymbol{\mu}_{h_i}, (1+r_{max})W^T\widetilde{\Sigma}W)$, where, $\widetilde{\Sigma}$ is the block diagonal covariance matrix as mentioned before. 

\subsection{First Linear Layer:} 
For the first linear layer, we reshape the input in a vector by
flattening both the channel and spatial dimensions. Similar
to the convolution layer, we concatenate $\{\mathbf{x}_i\}$ to be $\widetilde{\mathbf{x}}$,
whose covariance matrix has a block-diagonal structure.
Thus, the covariance matrix of the output pixel is $$\Sigma_{\mathbf{h}}=W^T\Sigma_{\mathbf{\widetilde{x}}} W$$
where $W$ is the learnable parameter and $\Sigma_{\mathbf{\widetilde{x}}}$ is the
block-diagonal covariance matrix of $\widetilde{\mathbf{x}}$ similar to the convolution layer.

{\it Special case:} From Obs. 1, we only need the largest two intensities to estimate the $\underline{p_{c_{\mathbf{x}}}}$. Thus if there is only one linear layer as the last layer in the entire network
(as in most ResNet like models), this can be further simplified: it needs
computing a $2\times 2$ covariance matrix instead of the $C \times C$ covariance matrix.

\subsection{Linear Layer:} 
If the network consists of multiple linear layers,
calculating the moments of the subsequent linear layers is performed differently.
Since it contains only 1-D inputs,
we can either treat it spatially or channel-wise.
In our setup, we consider it as along the channels. 

Let $\mathbf{x}\in \mathbf{R}^{N_i}$ be the input of the $i^{th}$ linear layer,
where $i>1$.
Assume, $\mathbf{x} \sim \mathcal{N}\left(\boldsymbol{\mu}^{(i)}_{\mathbf{x}}, \Sigma^{(i)}_{\mathbf{x}}\right)$.
Given the $i^{th}$ linear layer with parameter $(W_i, \mathbf{b}_i)$, with the output given by
$$\mathbf{h} = W_i^T\mathbf{x} + \mathbf{b}_i$$
$$\mathbf{h}  \sim \mathcal{N}\left(W_i^T\boldsymbol{\mu}^{(i)}_{\mathbf{x}}+\mathbf{b}_i, W_i^T\Sigma^{(i)}_{\mathbf{x}}W_i\right)$$ Here, $W_i \in \mathbf{R}^{N_{i}\times N_{i+1}}$ and $\mathbf{b}_i \in \mathbf{R}^{N_{i+1}}$ and $\mathbf{h} \in \mathbf{R}^{N_{i+1}}$. 

\subsection{Pooling Layer:} 
Recall that the input of a max pooling layer is $\{\mathbf{x}_i\}$ where each $\mathbf{x}_i\in \mathbf{R}^{N_{in}}$ and the index $i$ varies over the spatial dimension. Observe that as we identify each $\mathbf{x}_i$ by the respective distribution $\mathcal{N}\left(\boldsymbol{\mu}_{i}, \Sigma\right)$, applying max pooling over $\mathbf{x}_i$ essentially requires computing the maximum over $\left\{\mathcal{N}\left(\boldsymbol{\mu}_{i}, \Sigma\right)\right\}$. Thus, we restrict ourselves to average pooling. To be precise, with a kernel window $\mathbb{W}$ of size $k\times k$ with stride $k$ used for average pooling, the output of average pooling, denoted by $$\mathbf{h}\sim \mathcal{N}\left(\frac{1}{k^2} \sum_{\mathbf{x}_i\in \mathbb{W}} \boldsymbol{\mu}_{i}, \frac{\Sigma}{k^2} \right)$$

\subsection{Normalization Layer:}
For the normalization layer, given by $\mathbf{h} = (\mathbf{x} - \mu)/\sigma$, where $\mu,\sigma$ can be computed in different ways \cite{ioffe2015batch,ba2016layer,ulyanov2016instance}, we have 
$$\mathbf{h}  \sim \mathcal{N}\left(\frac{\boldsymbol{\mu}^{(i)}_{\mathbf{x}}-\mu}{\sigma}, \frac{\Sigma^{(i)}_{\mathbf{x}}}{\sigma^2} \right)$$ 
However, as the normalization layers often have large Lipschitz constant \cite{awais2020towards}, we remove those layers in this work.

\noindent{\bf Batch normalization: }
For the batch normalization, the mean and variance are computed within each mini-batch \cite{ioffe2015batch}. 
$$ \mu = \frac{1}{m} \sum_{i=1}^m x_i, \sigma^2 = \frac{1}{m} \sum_{i=1}^{m} (x_i-\mu)^2 $$
where $m$ is the size of the mini-batch. One thing to notice that $\mu,\sigma\in \mathbf{R}^{N}$ when $x_i\in \mathbf{R}^N$, $N$ is the channel size. 
In this setting, the way to compute the bounded distribution of output will need to be modified as following:
$$ \mathbf{h}  \sim \mathcal{N}\left(\frac{\boldsymbol{\mu}^{(i)}_{\mathbf{x}}-\mu}{\sigma}, \frac{\Sigma^{(i)}_{\mathbf{x}}}{\sigma \sigma^T } \right) $$
where ``$/$'' is the element-wise divide. $\sigma \sigma^T\in \mathbf{R}^{N\times N}$. And $\mu,\sigma$ will be computed dynamically as the new data being fed into the network.

\noindent{\bf Act normalization: }
As the performance for batch normalization being different between training and testing period, there are several other types of normalization layers. One of the recent one \cite{NEURIPS2018_d139db6a} proposed an act normalization layer where $\mu,\sigma$ are trainable parameters. These two parameters are initialed during warm-up period to compute the mean and variance of the training dataset. After initialization, these parameters are trained freely.

In this case, the update rule is the same as above. The only difference is that $\mu,\sigma$ do not depended on the data being fed into the network.

\subsection{Activation Layer:} 
\noindent{\bf ReLU: }
In \cite{bibi2018analytic,lee2019probact}, the authors introduced a way to compute the mean and variance after ReLU. 
Since ReLU is an element-wise operation, assume $x\sim \mathcal{N}(\mu,\sigma^2)$. After ReLU activation, the first and second moments of the output are given by:
\begin{align*}
    \mathbb{E}(\text{ReLU}(x))&=\frac{1}{2}\mu-\frac{1}{2}\mu\ \textsf{erf}(\frac{-\mu}{\sqrt{2}\sigma}) + \frac{1}{\sqrt{2\pi}}\sigma\exp(-\frac{\mu^2}{2\sigma^2})\\
    \text{var}(\text{ReLU}(x))&<\text{var}(x)
\end{align*}
Here, $\textsf{erf}$ is the Error function. We need to track an upper bound of the covariance matrix, so
we use $\text{ReLU}(x) \sim \mathcal{N}\left(\mu_a, \Sigma_a\right)$, where, 
\begin{align*}
    \mu_a &=\frac{1}{2}\mu-\frac{1}{2}\mu\ \textsf{erf}(\frac{-\mu}{\sqrt{2}\sigma}) + \frac{1}{\sqrt{2\pi}}\sigma\exp(-\frac{\mu^2}{2\sigma^2}), \\
    \Sigma_a &\preceq\Sigma
\end{align*}
\paragraph{Last layer/prediction:} 
Here, the last layer is the layer before softmax layer,
which represents the ``strength'' of the model for the label $l$. By Obs. 1, we have the
estimation of $$p_{c_\mathbf{x}}=\underline{p_{c_\mathbf{x}}}=\Phi(\frac{\boldsymbol{\mu}[{c_\mathbf{x}}]-\boldsymbol{\mu}[{\widetilde{c}}]}{\sqrt{\Sigma[{{c_\mathbf{x}},{c_\mathbf{x}}}]+\Sigma[{\widetilde{c},\widetilde{c}}]-2\Sigma[{{c_\mathbf{x}},\widetilde{c}}]}})$$ and $$p_{\widetilde{c}}=\overline{p_{\widetilde{c}}}=1-p_{c_\mathbf{x}}$$ By Theorem 1, the certified radius is 
\begin{align*}
	C_R &= \frac{\sigma}{2}(\Phi^{-1}(\underline{p_{c_\mathbf{x}}})-\Phi^{-1}(\overline{p_{\widetilde{c}}}))\\ 
	&= \sigma \frac{\boldsymbol{\mu}[c]-\boldsymbol{\mu}[{\widetilde{c}}]}{\sqrt{\Sigma[{{c_\mathbf{x}},{c_\mathbf{x}}}]+\Sigma[{\widetilde{c},\widetilde{c}}]-2\Sigma[{{c_\mathbf{x}},\widetilde{c}}]}}
\end{align*}

\noindent{\bf Other activation functions: }
Other than ReLU, there is no closed form to compute the mean after the activation function. One simplification can be: use the local linear function to approximate the activation function, as Gowal, et al. did \cite{gowal2018effectiveness}. The good piece of this approximation is that the output covariance matrix $\Sigma_a$ can be bounded by $\alpha^2 \Sigma$, where $\alpha$ is the largest slope of the linear function. 

Another direction is to apply the Hermite expansion \cite{lokhande2020generating} which will add more parameters but is theoretically sound. 

As ReLU is the most well-used activation function as well as the elegant closed form mean after the activation function, in this paper, we only focus on the ReLU layer and hold other activation functions into future work.

\section{Training Loss}
Now that we have defined the basic modules and the ways to propagate the distribution through these modules, we now
need to make the entire model trainable. In order to do that, we need to define an appropriate training loss which is the purpose of this section. In the spirit of  \cite{zhai2020macer}, the training loss contains two parts: the classification loss and the robustness loss, i.e., 
\begin{align*}
    l(g_\theta;\mathbf{x},y)&=l_C(g_\theta;\mathbf{x},y)+\lambda l_{C_R}(g_\theta;\mathbf{x},y)
\end{align*}

Let the distribution of the output of the last layer be $u_\theta(\mathbf{x})\sim\mathcal{N}(\boldsymbol{\mu},\Sigma)$. We need to output the expectation as the prediction, i.e.,  $\boldsymbol{\mu} = \mathbb{E}[u_\theta(\mathbf{x})]$. Similar to the literature, we use the softmax layer on the expectation to compute the cross-entropy of the prediction and the true label.

\begin{align*}
    l_C(g_\theta;\mathbf{x},y) = y \log(\text{softmax}(\mathbb{E}[u_\theta(\mathbf{x})]))
\end{align*}

The robustness loss measures the certified radius of each sample. In our case, we have the radius $$C_R=\sigma \frac{\boldsymbol{\mu}[c]-\boldsymbol{\mu}[{\widetilde{c}}]}{\sqrt{\Sigma[{c,c}]+\Sigma[{\widetilde{c},\widetilde{c}}]-2\Sigma[{c,\widetilde{c}}]}}$$ Thus, we want to maximize the certified radius which is equivalent to minimize the robustness loss.
\begin{align*}
    &l_{C_R}(g_\theta;\mathbf{x},y=c_{\mathbf{x}})\\
     &= \max(0, \Gamma - \sigma \frac{\boldsymbol{\mu}[c_{\mathbf{x}}]-\boldsymbol{\mu}[{\widetilde{c}}]}{\sqrt{\Sigma[{c_{\mathbf{x}},c_{\mathbf{x}}}]+\Sigma[{\widetilde{c},\widetilde{c}}]-2\Sigma[{c_{\mathbf{x}},\widetilde{c}}]}})
\end{align*}
where $c_{\mathbf{x}}$ is the true label, $\widetilde{c}$ is the second highest possible label. $\Gamma$ is the offset to control the certified radius to consider.

\section{Proof Of The Theorem}
\begin{theorem}
    \label{certified_radius}
    \cite{cohen2019certified} Let $f_\theta:\mathbf{R}^d\rightarrow \mathcal{Y}$ be any deterministic or random function, and let $\boldsymbol{\varepsilon} \sim \mathcal{N}(\mathbf{0},\sigma^2I)$. 
Let $g_\theta$ be the random smooth classifier defined as $g_\theta(\mathbf{x})=\argmax_{c\in\mathcal{Y}}\mathbb{P}(f_\theta(\mathbf{x}+\boldsymbol{\varepsilon})=c_{\mathbf{x}})$. Suppose $c_{\mathbf{x}},\widetilde{c}\in\mathcal{Y}$ and $\underline{p_{c_{\mathbf{x}}}},\overline{p_{\widetilde{c}}}\in [0,1]$ satisfy $\mathbb{P}(f_\theta(\mathbf{x}+\boldsymbol{\varepsilon})=c_{\mathbf{x}})\geq \underline{p_{c_\mathbf{x}}}\geq \overline{p_{\widetilde{c}}}\geq \max_{\widetilde{c}\neq c_{\mathbf{x}}}\mathbb{P}(f_\theta(\mathbf{x}+\boldsymbol{\varepsilon})=\widetilde{c})$. 
Then $g_\theta(\mathbf{x}+\boldsymbol{\delta})=c_{\mathbf{x}}$ 
for all $\|\boldsymbol{\delta}\|_2<C_R$, 
where $C_R=\frac{\sigma}{2}(\Phi^{-1}(\underline{p_{c_{\mathbf{x}}}})-\Phi^{-1}(\overline{p_{\widetilde{c}}}))$. 
\end{theorem} 
The symbol $\Phi$ denotes the CDF of the standard Normal distribution, and $\Phi$ and $\Phi^{-1}$
are involved because of smoothing the Gaussian perturbation $\boldsymbol{{\varepsilon}}$.

\begin{proof}
	To show that $g_\theta(\mathbf{x} + \boldsymbol{\delta} ) = c_{\mathbf{x}}$, it follows from the definition of $g_\theta$ that we need to show that $$ \mathbb{P}(f(\mathbf{x}+\boldsymbol{\delta} + \boldsymbol{\varepsilon})=c_{\mathbf{x}})> \max_{c'\neq c_{\mathbf{x}}}  \mathbb{P}(f(\mathbf{x}+\boldsymbol{\delta} + \boldsymbol{\varepsilon})=c') $$
	We will show $ \mathbb{P}(f(\mathbf{x}+\boldsymbol{\delta} + \boldsymbol{\varepsilon})=c_{\mathbf{x}})> \max_{c'\neq c_{\mathbf{x}}}  \mathbb{P}(f(\mathbf{x}+\boldsymbol{\delta} + \boldsymbol{\varepsilon})=c') $ for every class $c'\neq c_{\mathbf{x}}$. Fix one $c'$ without loss of generality.
	
	Define the random variables
	\begin{align*}
		X&:=\mathbf{x}+\boldsymbol{\varepsilon}=\mathcal{N}(\mathbf{x},\sigma^2 I)\\
		Y&:=\mathbf{x}+\boldsymbol{\boldsymbol{\delta}}+\boldsymbol{\varepsilon}=\mathcal{N}(\mathbf{x}+\boldsymbol{\delta},\sigma^2 I)
	\end{align*}
	
	We know that 
	$$ \mathbb{P}(f_\theta(X)=c_{\mathbf{x}})\geq \underline{p_{c_\mathbf{x}}},\ and\  \overline{p_{\widetilde{c}}}\geq \mathbb{P}(f_\theta(X)=c') $$
	
	We need to show $$ \mathbb{P}(f(Y)=c_{\mathbf{x}})>\mathbb{P}(f(Y)=c') $$ 
	
	Define the half-spaces:
	\begin{align*}
		A &:= \{z:\delta^T(z-x)\leq \sigma ||\delta|| \Phi^{-1}(\underline{p_{c_\mathbf{x}}}) \} \\
		B &:= \{z:\delta^T(z-x)\geq \sigma ||\delta|| \Phi^{-1}(1-\overline{p_{\widetilde{c}}}) \}
	\end{align*}
	
	We have $\mathbb{P}(X\in A)=\underline{p_{c_\mathbf{x}}}$. Therefore, we know $ \mathbb{P}(f(X)=c_\mathbf{x}) \leq \mathbb{P}(X\in A) $. Apply the Neyman-Pearson for Gaussians with different means Lemma \cite{neyman1933ix} with $h(z):=\mathbf{1}[f(z)=c_{\mathbf{x}}]$ to conclude:
	$$ \mathbb{P}(f(Y)= c_\mathbf{x} \geq \mathbb{P}(Y\in A) $$
	
	Similarly, we have $\mathbb{P}(X\in B)=\overline{p_{\widetilde{c}}}$, then we know $\mathbb{P}(f(X)=c')\leq \mathbb{P}(X\in B)$. Use the Lemma \cite{neyman1933ix} again with $h(z):=\mathbf{1}[f(z)=c']$:
	$$ \mathbb{P}(f(Y)= c') \leq \mathbb{P}(Y\in B) $$
	
	Then to guarantee $\mathbb{P}(f(Y)=c_{\mathbf{x}})>\mathbb{P}(f(Y)=c')$, it suffices to show that $\mathbb{P}(Y\in A) > \mathbb{P}(Y\in B)$. Thus, if and only if 
	$$ ||\boldsymbol{\delta}|| < \frac{\sigma}{2} (\Phi^{-1}(\underline{p_{c_\mathbf{x}}}) - \Phi^{-1}( \overline{p_{\widetilde{c}}} )) $$
\end{proof}

\section{Proof Of The Observations}
  \begin{proposition}
	\label{Gaussian_estimate}
	Let $u_{\theta}$ denote
        the network without the last softmax layer, i.e., the full
        neural network
        can be written as  $$f_\theta(\mathbf{x})=\argmax_{c\in\mathcal{Y}}\text{ softmax}(u_\theta(\mathbf{x}))$$ Let $C = |\mathcal{Y}|$ and assume  $$u_\theta(\mathbf{x}) \sim \mathcal{N}(\boldsymbol{\mu},\Sigma)$$ where $\boldsymbol{\mu} \in \mathbf{R}^C$ and $\Sigma \in \mathbf{R}^{C\times C}$. Then the estimation of $\underline{p_{c_{\mathbf{x}}}}$ is $$\underline{p_{c_{\mathbf{x}}}}= \Phi(\frac{\boldsymbol{\mu}[{c_{\mathbf{x}}}]-\boldsymbol{\mu}[{\widetilde{c}}]}{\sqrt{\Sigma[{c_{\mathbf{x}},c_{\mathbf{x}}}]+\Sigma[{\widetilde{c},\widetilde{c}}]-2\Sigma[{c_{\mathbf{x}},\widetilde{c}}]}}),$$ where  
        \begin{align*}
        	c_{\mathbf{x}}=\displaystyle\argmax_{c\in\mathcal{Y}}\mathbb{P}\left(u_\theta(\mathbf{x})=c \right)\\
        \widetilde{c}=\displaystyle\argmax_{{c}\in\mathcal{Y},{c}\neq c_{\mathbf{x}}}\mathbb{P}\left(u_\theta(\mathbf{x})={c}\right)
        \end{align*}
        \end{proposition}
 
 \begin{proof}
   Let us call the strength of two labels to be $u^{c_\mathbf{x}}, u^{\widetilde{c}}$. We have $$u^{c_\mathbf{x}} \sim\mathcal{N}(\boldsymbol{\mu}[{c_{\mathbf{x}}}],\Sigma[{c_{\mathbf{x}},c_{\mathbf{x}}}]), u^{\widetilde{c}}\sim\mathcal{N}(\boldsymbol{\mu}[{\widetilde{c}}],\Sigma[{\widetilde{c},\widetilde{c}}]).$$
   The cross-correlations $E_{c_\mathbf{x}, \widetilde{c}}=\Sigma[{c_{\mathbf{x}},\widetilde{c}}]$. 
   Thus, $$\underline{p_{c_{\mathbf{x}}}}=\mathbb{P}(u^{c_\mathbf{x}} > u^{\widetilde{c}})=\mathbb{P}(u^{c_\mathbf{x}} - u^{\widetilde{c}} >0).$$
   Let $\eta=u^{c_\mathbf{x}} - u^{\widetilde{c}} \in \mathbf{R}$, then $$\eta\sim\mathcal{N}(\boldsymbol{\mu}[{c_{\mathbf{x}}}]-\boldsymbol{\mu}[{\widetilde{c}}], \Sigma[{c_{\mathbf{x}},c_{\mathbf{x}}}]+\Sigma[{\widetilde{c},\widetilde{c}}]-2E_{c_\mathbf{x}, \widetilde{c}})$$
 	Then, to compute $\underline{p_{c_{\mathbf{x}}}}$,
 	\begin{align*}
    \underline{p_{c_{\mathbf{x}}}} &= \mathbb{P}(u^{c_\mathbf{x}} > u^{\widetilde{c}})\\
    &=\mathbb{P}(\eta>0)\\
    &=\mathbb{P}\bigl(\frac{\eta-\boldsymbol{\mu}[{c_{\mathbf{x}}}]+\boldsymbol{\mu}[{\widetilde{c}}]}{\sqrt{\Sigma[{c_{\mathbf{x}},c_{\mathbf{x}}}]+\Sigma[{\widetilde{c},\widetilde{c}}]-2E_{c_\mathbf{x}, \widetilde{c}}}}>\\
    & \frac{-\boldsymbol{\mu}[{c_{\mathbf{x}}}]+\boldsymbol{\mu}[{\widetilde{c}}]}{\sqrt{\Sigma[{c_{\mathbf{x}},c_{\mathbf{x}}}]+\Sigma[{\widetilde{c},\widetilde{c}}]-2E_{c_\mathbf{x}, \widetilde{c}}}}\bigr)\\
    &=\mathbb{P}(\eta'>\frac{-\boldsymbol{\mu}[{c_{\mathbf{x}}}]+\boldsymbol{\mu}[{\widetilde{c}}]}{\sqrt{\Sigma[{c_{\mathbf{x}},c_{\mathbf{x}}}]+\Sigma[{\widetilde{c},\widetilde{c}}]-2E_{c_\mathbf{x}, \widetilde{c}}}})\\
    &\ \ \text{where}\ \eta'=\frac{\eta-\boldsymbol{\mu}[{c_{\mathbf{x}}}]+\boldsymbol{\mu}[{\widetilde{c}}]}{\sqrt{\Sigma[{c_{\mathbf{x}},c_{\mathbf{x}}}]+\Sigma[{\widetilde{c},\widetilde{c}}]-2E_{c_\mathbf{x}, \widetilde{c}}}}\sim\mathcal{N}(0,1)\\
    &=1-\Phi(\frac{-\boldsymbol{\mu}[{c_{\mathbf{x}}}]+\boldsymbol{\mu}[{\widetilde{c}}]}{\sqrt{\Sigma[{c_{\mathbf{x}},c_{\mathbf{x}}}]+\Sigma[{\widetilde{c},\widetilde{c}}]-2E_{c_\mathbf{x}, \widetilde{c}}}})\\
    &=\Phi(\frac{\boldsymbol{\mu}[{c_{\mathbf{x}}}]-\boldsymbol{\mu}[{\widetilde{c}}]}{\sqrt{\Sigma[{c_{\mathbf{x}},c_{\mathbf{x}}}]+\Sigma[{\widetilde{c},\widetilde{c}}]-2E_{c_\mathbf{x}, \widetilde{c}}}}), \Phi:c.d.f.\ of\ \mathcal{N}(0,1)
\end{align*}
 \end{proof}

\begin{proposition}
\label{prop:identical}
	With the input perturbation $\boldsymbol{\varepsilon}$ to be identical along the spatial dimension and without nonlinear activation function, for $q^{th}$ hidden convolution layer with $\{\mathbf{h}_i\}_{i=1}^{M_q}$ output pixels, we have $\Sigma_{\mathbf{h}_i}^{(q)} = \Sigma_{\mathbf{h}_j}^{(q)},\forall i, j\in \{1, \cdots, M_q\}$.
\end{proposition}

\begin{proof}
	Let us prove by induction. In the first layer, there is no cross-correlation between pixels, while the input perturbation is given to be identical. Thus, identical assumption is true for the first layer.\\
	Assume in the layer $1$ to $q$, the second moments are all identical within the layer. We need to prove that for layer $q+1$, this property also holds. Consider two pixels in layer $q+1$, $\mathbf{h}_{i}^{(q+1)}$ and $\mathbf{h}_{j}^{(q+1)}$. By definition of convolution, we know $\mathbf{h}_i^{(q+1)}=\sum W_k \mathbf{h}_{k}^{(q)}$ where $\mathbf{h}_{k}^{(q)}$ is the view field of $\mathbf{h}_{i}^{(q+1)}$ and $W$ is the shared kernel. So the covariance matrix of $\mathbf{h}_{i}^{(q+1)}$ is $$\Sigma_{\mathbf{h}_i}^{(q+1)}=\sum_k W_k^T \Sigma_{\mathbf{h}_k}^{(q)} W_k + \sum_{m,n} W_m^T E_{m,n}^{(q)} W_n$$ By the identical assumption, the first components are the same for $\mathbf{h}_{i}^{(q+1)}$ and $\mathbf{h}_{j}^{(q+1)}$. The only problematic component is the second piece which requires computing $E_{m,n}^{(q)}$. For $E_{m,n}^{(q)}$, we can compute recursively that 
	\begin{align*}
		E_{m,n}^{(q)}&=\mathbb{E}[(\mathbf{h}_m^{(q)}-\mathbb{E}[\mathbf{h}_m^{(q)}])^T (\mathbf{h}_n^{(q)}-\mathbb{E}[\mathbf{h}_n^{(q)}])]\\
		&=g^{(q)}(\Sigma^{(q-1)}, E_{m',n'}^{(q-1)}, W^{(q-1)})
	\end{align*}
	where $g^{(q)}(\cdot)$ is a determined function. Thus, the only problematic piece is also $E_{m',n'}^{(q-1)}$. If we keep doing until we reach the first layer, we will have $$\Sigma^{(q)}=f^{(q)}\circ f^{(q-1)} \circ ... f^{(1)}(\Sigma^{(1)})$$ with $f(\cdot)$ is the function depending on $g(\cdot)$, which is a determined value across all the pixels. This assumption depends on the linearity of the network, which will in turn separate the first and second order moments when passing through the network.\\
	As a simplification case, we can assume that the pixels in the same layer are independent, which implies $E_{m,n}^{(q)}$ to be $0$ everywhere. This case is suitable when we apply Theorem 2 in the main paper to each convolution layer.
\end{proof}
One thing to notice is the activation function. In our method, in order to keep the identical assumption, we need to make the second moments after activation function to be identical when given the identical input. Also, we will need to keep an upper bound so that the whole method is theoretically sound. Due to the fact that the ReLU will reduce the variance of the variable, we assign the output covariance to be the same as the input covariance. Thus, the identical assumption is kept through all the layers. 

\section{Other Results Of Our Method}
We also performed the proposed method on Cifar-100 \cite{krizhevsky2009learning}. This dataset is more challenging than Cifar-10 as the number of classes increases from 10 to 100. The results are shown in Table. \ref{tab:noisy_cifar100}. 

\newcolumntype{g}{>{\columncolor{Gray}}c}
\begin{table}[h]
\caption {\footnotesize Average test accuracy on pair-flipping with noisy rate $45\%$ over the last ten epochs of Cifar-100. We show the results of Bootstrap\cite{reed2014training}, S-model\cite{goldberger2016training}, Decoupling\cite{malach2017decoupling}, MentorNet\cite{jiang2017mentornet}, Co-teaching\cite{han2018co}, Trunc $\mathcal{L}_q$\cite{zhang2018generalized}, and Ours.} \label{tab:noisy_cifar100} 
\centering
\scalebox{0.58}{
\begin{tabular}{c|c g c g c g c}
\topline\myrowcolour
Method & Bootstrap & S-model & Decoupling & MentorNet & Co-teaching & Trunc $\mathcal{L}_q$ & Ours \\
\hline
mean  & 0.321 & 0.218 & 0.261 & 0.316 & 0.348 &\highest{0.477}  & 0.346\\
    std                        & 3.0$e-3$ & 8.6$e-3$ & 0.3$e-3$ & 5.1$e-3$ & 0.7$e-3$ & 6.9$e-3$ & 0.6$e-3$ \\ 
\bottomline 
\end{tabular}
}
\vspace*{-1em}
\end{table}

\section{Strengths And Limitations Of Our Method}

The biggest benefit of our method is the training time, which is also the
main focus of the paper. As in the main paper, we have shown that our methods can be $5\times$ faster on Cifar-10, etc. dataset, with a comparable ACR as MACER. In real-world applications, training speed is
an important consideration. Thus, our method is a cheaper substitute of MACER with a marginal performance compromise. 

On the other hand, we also need to discuss that under which circumstances our method does not perform well.
The first case is when the network is extremely deep, e.g., Resnet-101.
Due to the nature of upper bound, the estimation of the second moments tends to become looser as the network grows
deeper. Thus, this will lead to a looser estimation of the distribution of the last layer and
the robustness estimation would be less meaningful.
Another minor weakness is when the input perturbation is large, for example $\sigma=1.0$.
As shown in the main paper, the ACR drops from 0.56 to 0.52 on ImageNet when the
noise perturbation increases from $\sigma=0.5$ to $\sigma=1.0$. The main reason relates to 
the assumption of a Gaussian distribution. As the perturbation grows larger, the number of channels,
by the central limit theorem, should also be larger to satisfy the Gaussian distribution.
Thus, when the network structure is fixed, there is an inherent limit on the input perturbation. 

We note that to perform a fair comparison, we run Cohen's \cite{cohen2019certified},
MACER \cite{zhai2020macer}, and our method based on the PreActResnet-18 for the Table. 3 in the main paper.
Since the network is shallower than the original MACER paper, the performance numbers reported here are lower.
%Also, this network cannot handle
As described above, a large perturbation $\sigma=1.0$ leads to small drop in performance.
Thus, almost for all three methods, the ACR for $\sigma=1.0$ is worse than
the one for $\sigma=0.5$ on ImageNet. For Places365 dataset, the results are slightly better on $\sigma=1.0$
than $\sigma=0.5$. 

Also, similar as the main paper, we statistically test the variance based on MC sampling and our upper bound tracking method. The results are shown in Table. \ref{tab:stat_supp}. As one can see that when the network gets deeper, the upper bound tends to be looser. But in most case, the upper bound is around 3 times higher than the sample-based variance, which is affordable in the real-world application.

\begin{table}[h]
 \caption {\footnotesize Statistics for different layers of MC sampling and our upper bound tracking method for deeper network.} \label{tab:stat_supp} 
 \centering
 \scalebox{0.65}{
 \begin{tabular}{c|ccccc}
 \topline\myrowcolour
 Layer number & 1 & 9 & 17 & 25 & 33\\
 \hline
% MC (1000 samples) 	& 0.637 & 0.542 & 0.745 & 1.740 & 0.276 \\
% Upper bound 		& 0.815 & 2.668 & 3.557 & 6.723 & 1.968 \\
  MC (1000 samples) 	& 0.595 & 0.782 & 2.751 & 7.692 & 0.712 \\
 	Upper bound 		& 0.685 & 3.231 & 5.583 & 22.546 & 3.960 \\
% 	Upper bound 		& 0.630 & 2.232 & 2.790 & 8.663 & 1.210 \\
\bottomline 
 \end{tabular}
 }
 \end{table}

\begin{table*}[ht]
\caption{\footnotesize Consider a 1-D convolution network with kernel size $k=3$. The blue dots are the nodes as well as the covariance matrices $\Sigma$ to be computed. The red arrows are the cross-correlation $E$ between two nodes. }
\label{tab:example_exp}
\begin{tabular}{p{\columnwidth}|p{\columnwidth}}
\topline
	\parbox[c]{1.0\columnwidth}{
      \includegraphics[width=1.0\columnwidth]{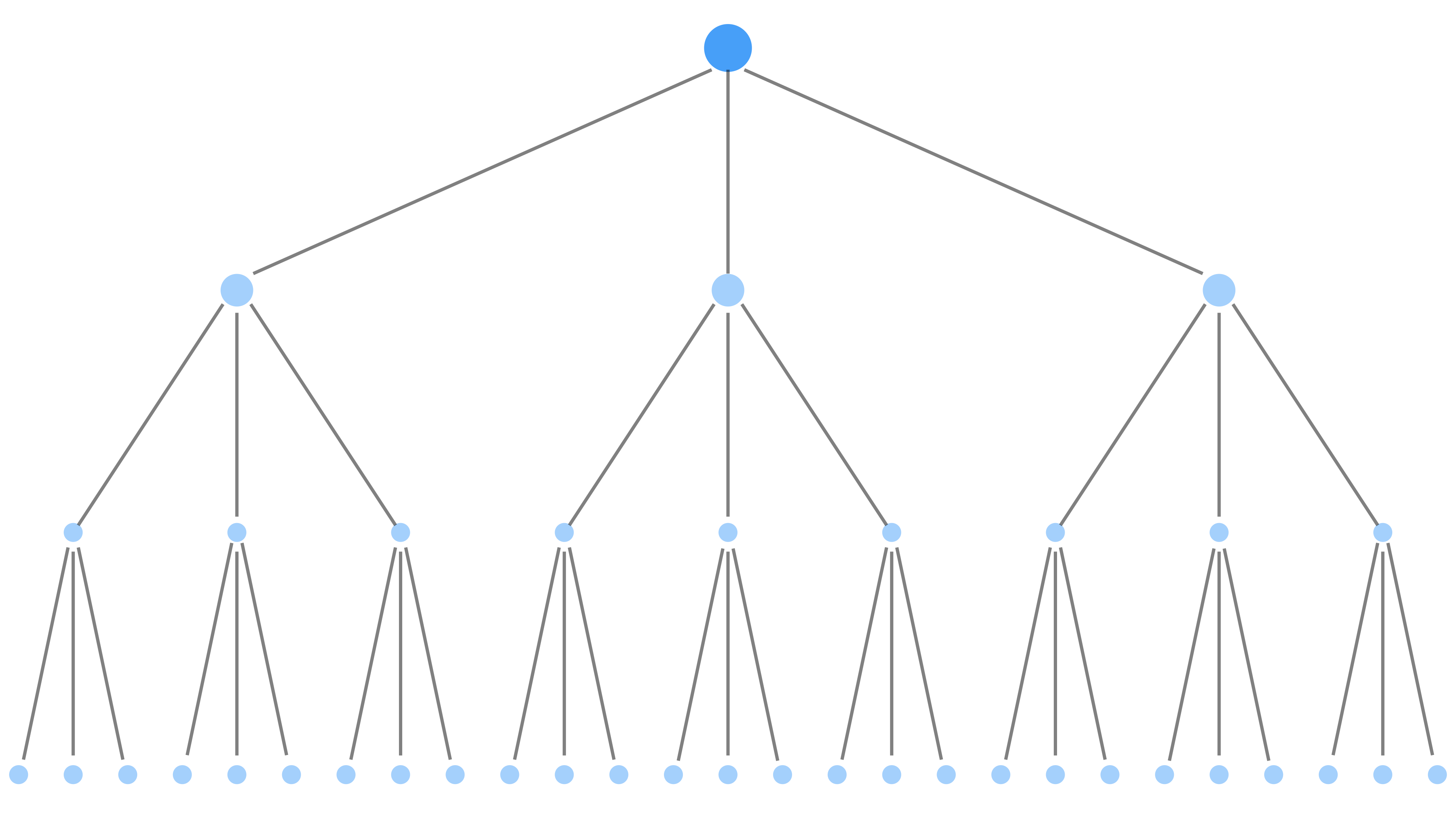}} & In the last layer, \newline
      	$\begin{aligned} p_{c_\mathbf{x}}=\underline{p_{c_\mathbf{x}}}=\Phi\left(\frac{\boldsymbol{\mu}[{c_\mathbf{x}}]-\boldsymbol{\mu}[{\widetilde{c}}]}{\sqrt{\Sigma[{{c_\mathbf{x}},{c_\mathbf{x}}}]+\Sigma[{\widetilde{c},\widetilde{c}}]-2\Sigma[{{c_\mathbf{x}},\widetilde{c}}]}}\right) \end{aligned}$,\newline
      We need to compute 1 $\Sigma^{(i)}$\\
    \hline
    \parbox[c]{1.0\columnwidth}{
      \includegraphics[width=1.0\columnwidth]{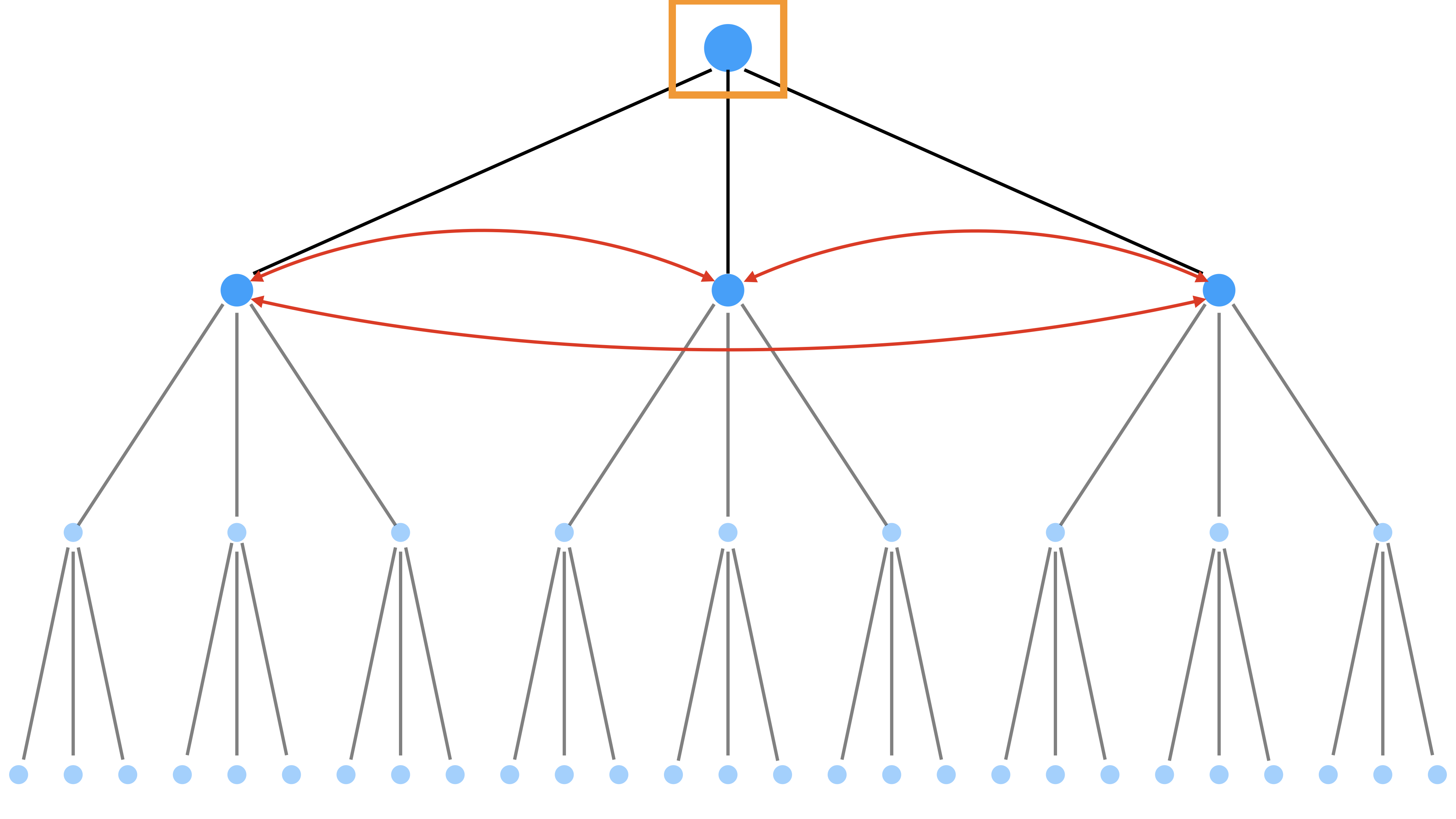}} & To compute the $\Sigma^{(i)}$ above (highlighted in orange box), in the second last layer\newline
      	$\begin{aligned} \Sigma^{(i)} &= \mathbb{E}(W^T{X^{(i-1)}}^T {X^{(i-1)}} W)\\ &= \sum_{m=1}^{3} W_m^T \Sigma_m^{(i-1)} W_m + \sum_{m,n=1,m\neq n} ^{3} W_m^T E_{mn}^{(i-1)} W_n \end{aligned}$\newline
      	We need to compute 3 $\Sigma^{(i-1)}$ and 3 $E^{(i-1)}$.
      	\\
    \hline
    \parbox[c]{1.0\columnwidth}{
      \includegraphics[width=1.0\columnwidth]{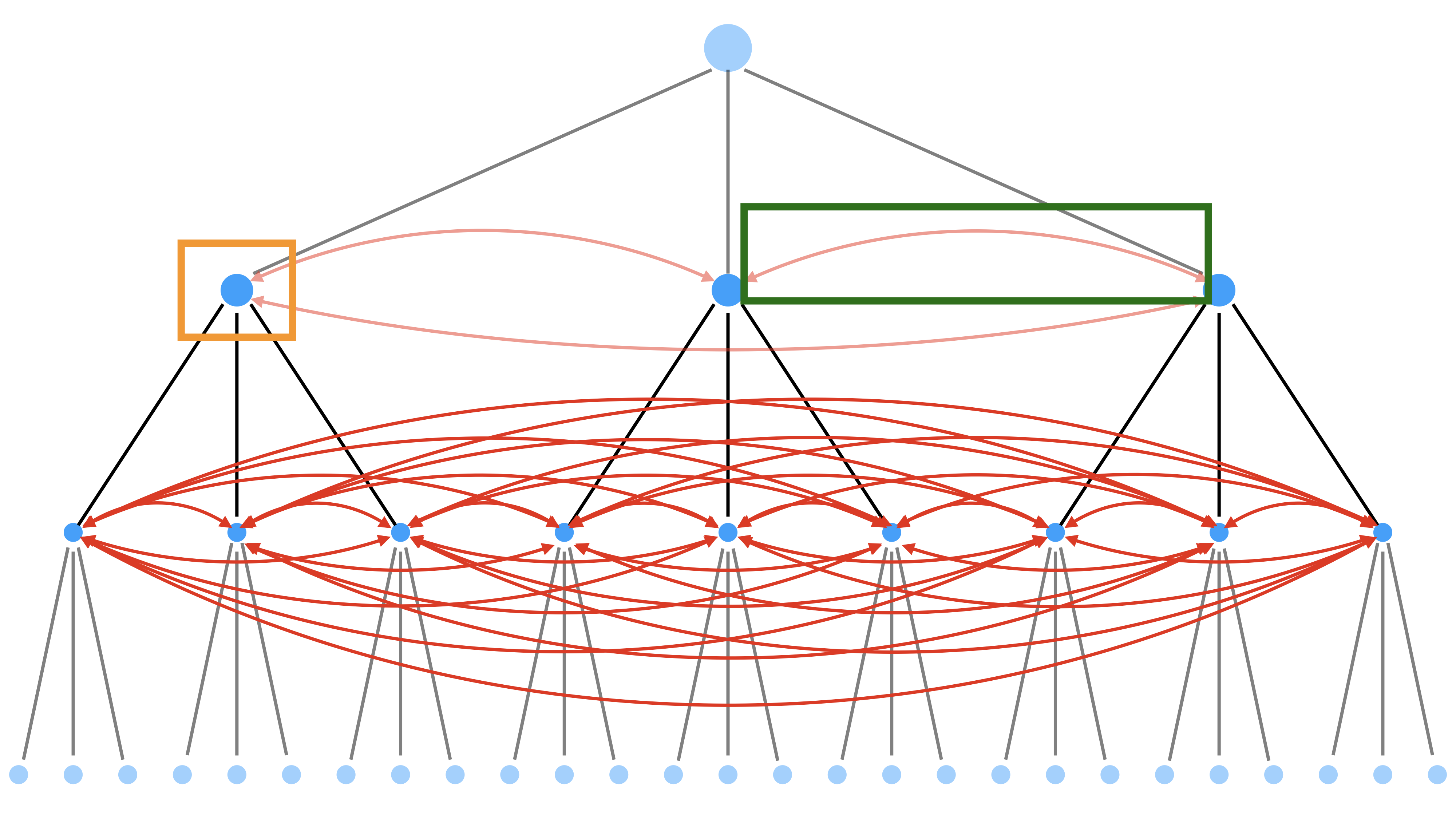}} & To compute the $\Sigma^{(i-1)}$ above (highlighted in orange box), in the third last layer, \newline
      $\begin{aligned} \Sigma^{(i-1)} &= \sum_{m=1}^{3} W_m^T \Sigma_m^{(i-2)} W_m + \sum_{m,n=1,m\neq n} ^{3} W_m^T E_{mn}^{(i-2)} W_n \end{aligned}$\newline
      To compute the $E_{12}^{(i-1)}$ above (highlighted in green box)\newline
      $\begin{aligned} E_{12}^{(i-1)} &= \mathbb{E}({x_1^{(i-1)}}^T x_2^{(i-1)}) = \mathbb{E}({W^T {X_1^{(i-2)}}^T X_2^{(i-2)} W}) \\ &= \sum_{m,n=3}^{8} \tilde{W}^T E_{mn}^{(i-2)} \tilde{W'} \end{aligned}$, where $\tilde{W}, \tilde{W'}$ are the weights depending on the position of $m,n$.\newline
      We need to compute $3^2$ $\Sigma^{(i-2)}$ and $\frac{1}{2}(3^2+1)3^2$ $E^{(i-2)}$.
      
      \\
    \hline
    \parbox[c]{1.0\columnwidth}{
      \includegraphics[width=1.0\columnwidth]{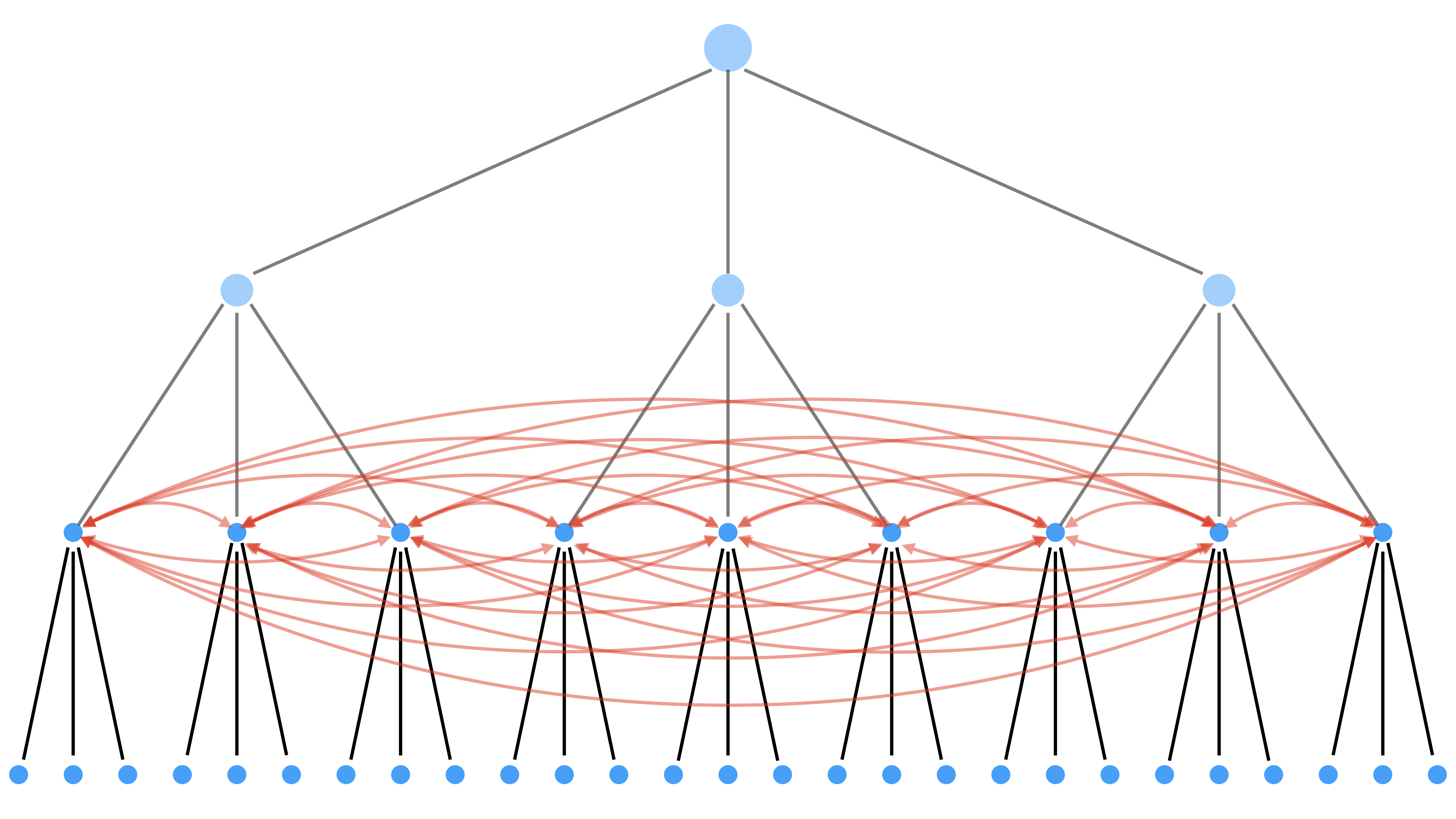}} & In the $(q+1)^{th}$ last layer, 
      
      We need to compute $k^q$ $\Sigma^{(i-q)}$ and $\frac{1}{2}(1+k^q)k^q$ $E^{(i-q)}$. \\
    \bottomline  
\end{tabular}
	
\end{table*}

\begin{table*}[ht]
\caption{\footnotesize Consider a 1-D convolution network with kernel size $k=3$ with overlapping. The blue dots are the nodes as well as the covariance matrices $\Sigma$ to be computed. The red arrows are the cross-correlation $E$ between two nodes. }
\label{tab:example_overlap}
\begin{tabular}{p{\columnwidth}|p{\columnwidth}}
\topline
	\parbox[c]{1.0\columnwidth}{
      \includegraphics[width=1.0\columnwidth]{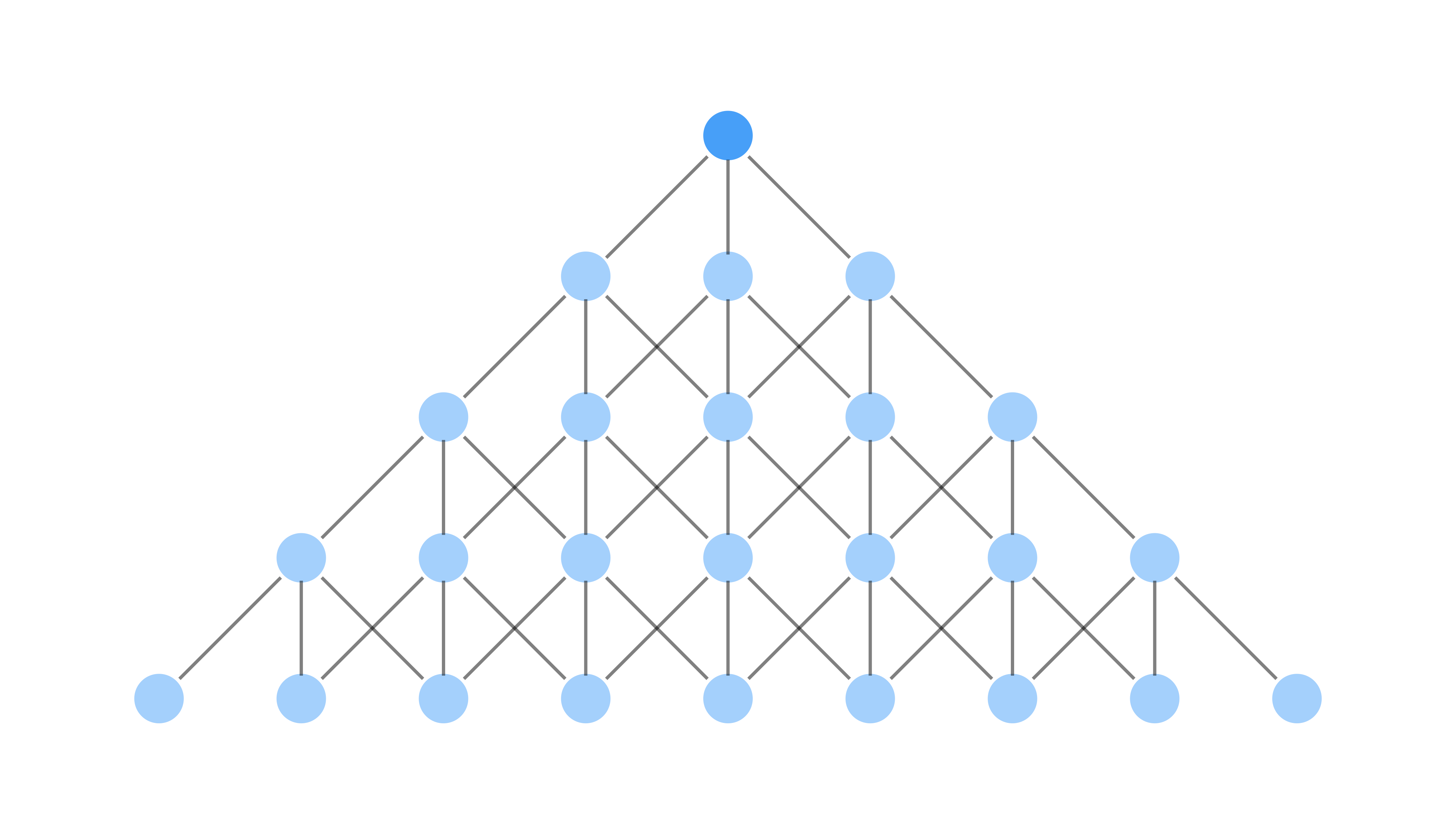}} & In the last layer, 
      
      	$\begin{aligned} p_{c_\mathbf{x}}=\underline{p_{c_\mathbf{x}}}=\Phi\left(\frac{\boldsymbol{\mu}[{c_\mathbf{x}}]-\boldsymbol{\mu}[{\widetilde{c}}]}{\sqrt{\Sigma[{{c_\mathbf{x}},{c_\mathbf{x}}}]+\Sigma[{\widetilde{c},\widetilde{c}}]-2\Sigma[{{c_\mathbf{x}},\widetilde{c}}]}}\right) \end{aligned}$,

      We need to compute 1 $\Sigma^{(i)}$\\
    \hline
    \parbox[c]{1.0\columnwidth}{
      \includegraphics[width=1.0\columnwidth]{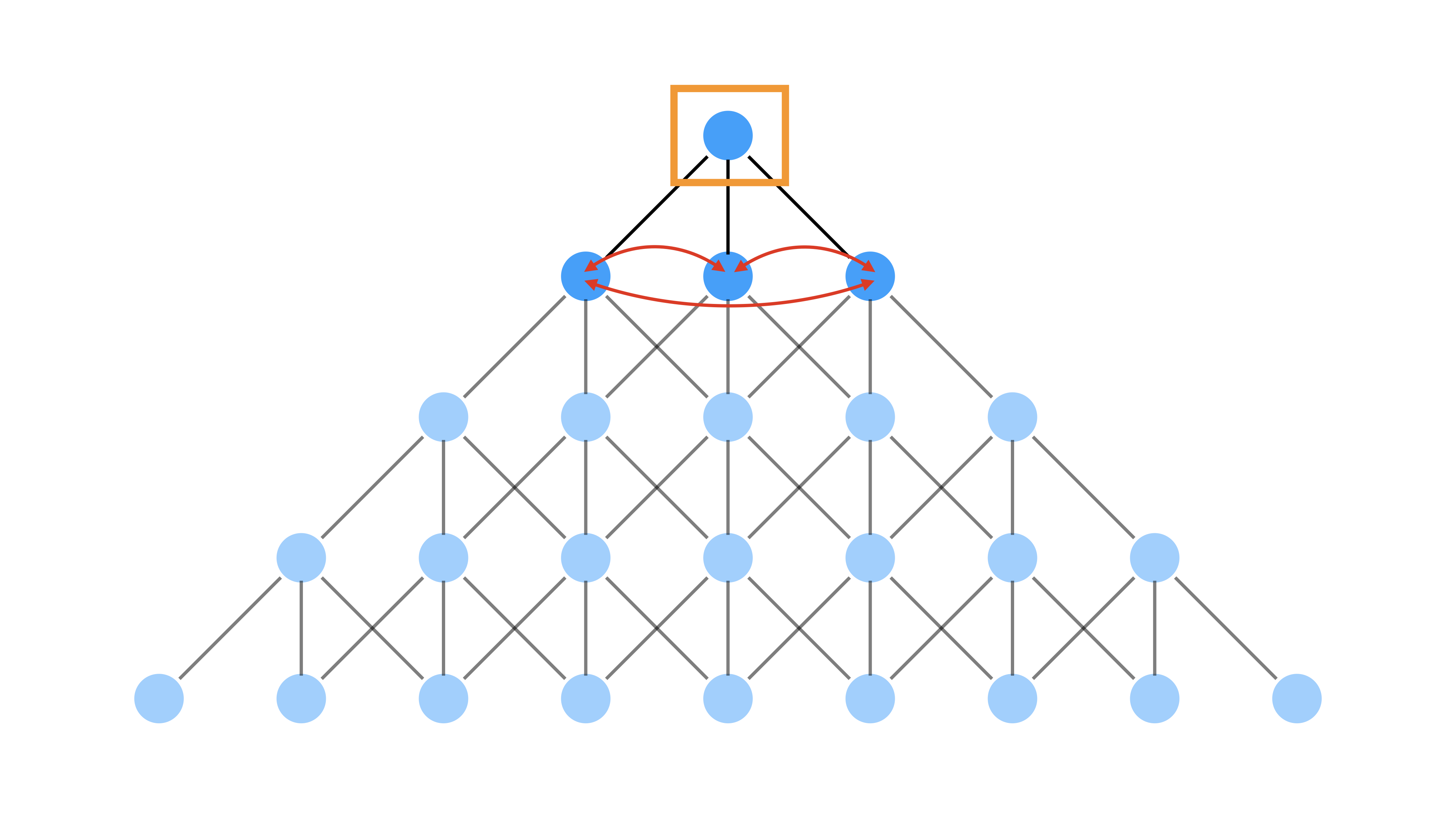}} & To compute the $\Sigma^{(i)}$ above (highlighted in orange box), in the second last layer
      
      	$\begin{aligned} \Sigma^{(i)} &= \mathbb{E}(W^T{X^{(i-1)}}^T {X^{(i-1)}} W)\\ &= \sum_{m=1}^{3} W_m^T \Sigma_m^{(i-1)} W_m + \sum_{m,n=1,m\neq n} ^{3} W_m^T E_{mn}^{(i-1)} W_n \end{aligned}$
      	We need to compute 3 $\Sigma^{(i-1)}$ and 3 $E^{(i-1)}$.
      	\\
    \hline
    \parbox[c]{1.0\columnwidth}{
      \includegraphics[width=1.0\columnwidth]{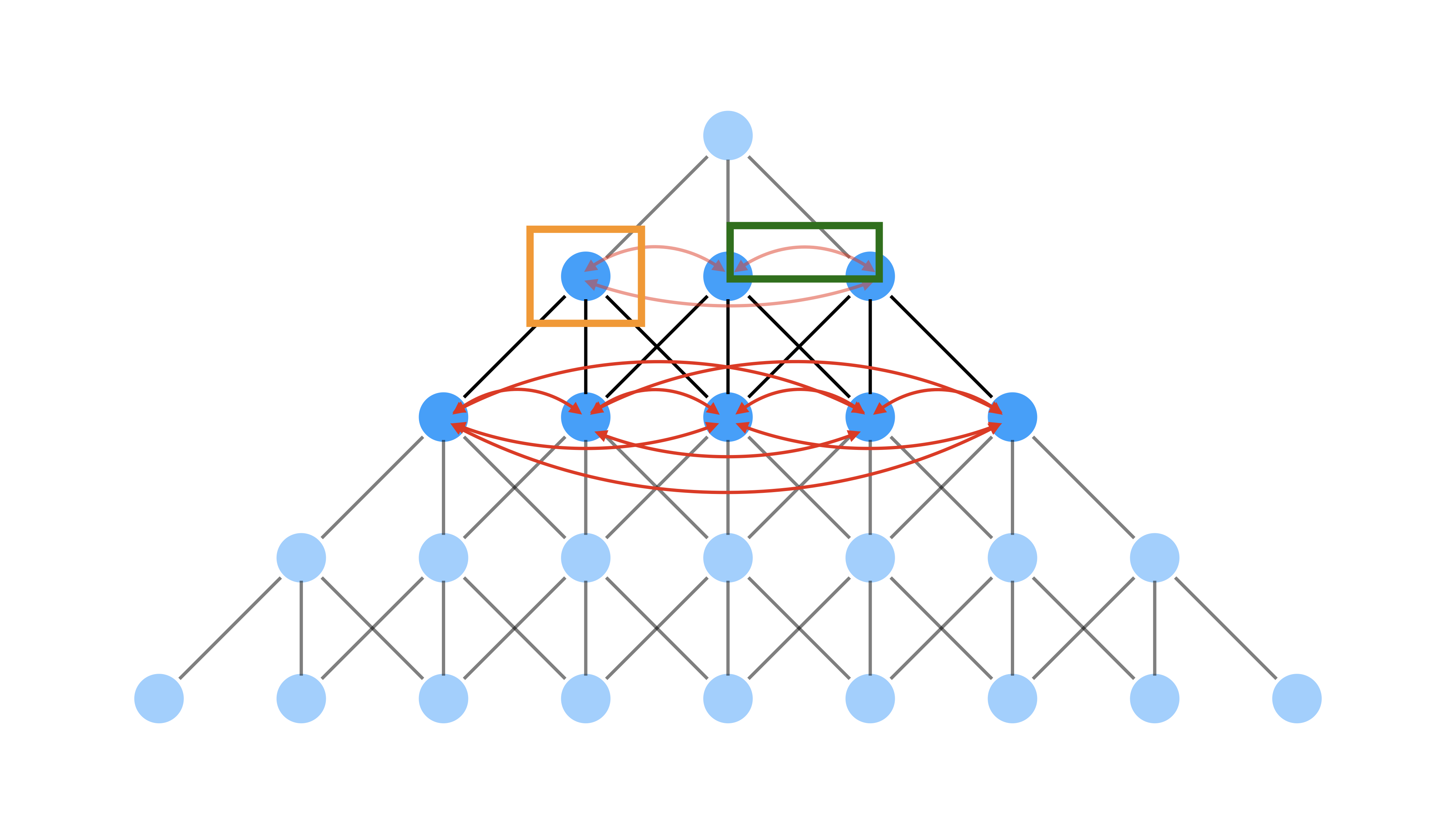}} & To compute the $\Sigma^{(i-1)}$ above (highlighted in orange box), in the third last layer, 
            
      $\begin{aligned} \Sigma^{(i-1)} &= \sum_{m=1}^{3} W_m^T \Sigma_m^{(i-2)} W_m + \sum_{m,n=1,m\neq n} ^{3} W_m^T E_{mn}^{(i-2)} W_n \end{aligned}$
      
      To compute the $E_{12}^{(i-1)}$ above (highlighted in green box) 
      
      $\begin{aligned} E_{12}^{(i-1)} &= \mathbb{E}({x_1^{(i-1)}}^T x_2^{(i-1)}) = \mathbb{E}({W^T {X_1^{(i-2)}}^T X_2^{(i-2)} W}) \\ &= \sum_{m,n=1}^{4} \tilde{W}^T E_{mn}^{(i-2)} \tilde{W'} \end{aligned}$, where $\tilde{W}, \tilde{W'}$ are the weights depending on the position of $m,n$.
     
      We need to compute $(3-1)\times 2+1$ $\Sigma^{(i-2)}$ and $\frac{1}{2}((3-1)\times 2+1+1)((3-1)\times 2+1)$ $E^{(i-2)}$.
      
      \\
    \hline
    \parbox[c]{1.0\columnwidth}{
      \includegraphics[width=1.0\columnwidth]{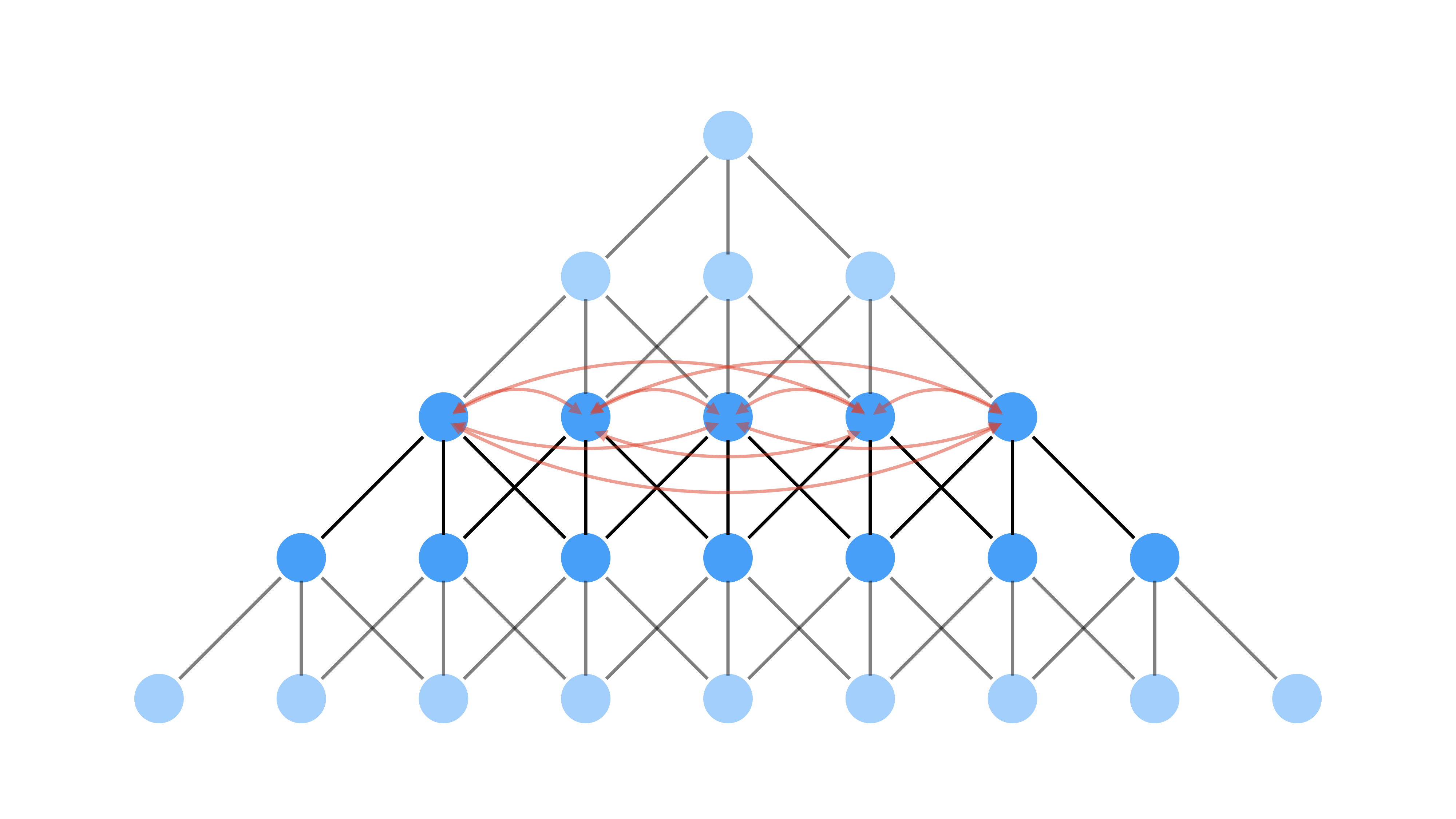}} & In the $(q+1)^{th}$ last layer, 
      
      We need to compute $(k-1)q+1$ $\Sigma^{(i-q)}$ and  $\frac{1}{2}((k-1)q+1+1)((k-1)q+1)$ $E^{(i-q)}$. \\
    \bottomline  
\end{tabular}
	
\end{table*}

\end{document}